%% file: writeup.tex
\documentclass[11pt,letter]{article}
\pdfoutput=1
\usepackage{amsmath, amsthm}
\usepackage{amssymb}
\usepackage{graphicx}
\usepackage{verbatim}
\usepackage{setspace}
\usepackage{listings}
\usepackage{xfrac}
\usepackage{subcaption}
\usepackage{thmtools}
\usepackage{thm-restate}
\usepackage{hyperref}
\usepackage{xcolor}
\definecolor{DarkGreen}{rgb}{0.1,0.5,0.1}
\definecolor{DarkRed}{rgb}{0.5,0.1,0.1}
\definecolor{DarkBlue}{rgb}{0.1,0.1,0.5}
\definecolor{Black}{rgb}{0.0,0.0,0.0}
\usepackage{hyperref}
\hypersetup{
    colorlinks=true,       
    linkcolor=DarkBlue,          
    citecolor=DarkBlue,        
    filecolor=DarkBlue,      
    urlcolor=DarkBlue,          
    pdftitle={Delayed Impact of Fair Machine Learning},
    pdfauthor={Lydia T. Liu, Sarah Dean, Esther Rolf, Max Simchowitz, Moritz Hardt},
}
\usepackage{color}
\usepackage{siunitx}
\usepackage[linesnumbered,ruled,vlined ]{algorithm2e}
\usepackage{mathtools}
\usepackage{natbib}

\graphicspath{{../experiments/figures/}{./}{figures}}

\theoremstyle{definition}

{
	\theoremstyle{plain}
	\newtheorem{asm}{Assumption}
}
\theoremstyle{plain}
\newtheorem{thm}{Theorem}[section]

\newtheorem{cor}[thm]{Corollary}
\declaretheorem[name=Lemma,numberwithin=section]{lem}

\newtheorem{prop}[thm]{Proposition}

\theoremstyle{definition}
\newtheorem{defn}{Definition}[section]

\newtheorem{exmp}{Example}[section]

\theoremstyle{remark}

\newcommand{\remove}[1]{}

\allowdisplaybreaks
\usepackage{fullpage}
	
\include{macros}

\title{\textbf{Delayed Impact of Fair Machine Learning}}
\author{
\and Lydia T.~Liu\thanks{Department of Electrical Engineering and Computer Sciences, University of California, Berkeley}
\and Sarah Dean\footnotemark[1]
\and Esther Rolf\footnotemark[1]
\and Max Simchowitz\footnotemark[1]
\and Moritz Hardt\footnotemark[1]}

\begin{document}

\maketitle

\begin{abstract}
\input{abstract}
\end{abstract}

\input{01_intro}

\input{02_problem_setting}

\input{03_results}
\input{06_main_technical_material}
\input{07_proofs_of_theorems}

\input{05_experiments}

\input{06_conclusion}
\input{acknow}

\pagebreak
\bibliographystyle{plainnat}
\bibliography{mybib.bib}

\begin{appendix}
\input{10_threshold_policies_appendix}
\input{10_characterizations}

\input{10_characterize_soft}

\input{0_appendix_main_results_proofs}
\end{appendix}

\end{document}

%% file: macros.tex
\newcommand{\calX}{\mathcal{X}}

\newcommand{\calT}{\mathcal{T}}

\newcommand{\calS}{\mathcal{S}}

 \newcommand{\accrate}{\beta}
  \newcommand{\accratebar}{\overline{\accrate}}

 \newcommand{\accratea}{\accrate\supa}
 \newcommand{\accrateb}{\accrate\supb}
 \newcommand{\accratej}{\accrate\supj}

  \newcommand{\tpr}{\mathrm{TPR}}

 \newcommand{\opp}{\tpr}
 \newcommand{\oppa}{\opp\supa}
 \newcommand{\oppb}{\opp\supb}
 \newcommand{\oppj}{\opp\supj}

\newcommand{\boldz}{\boldsymbol{z}}

\newcommand{\Constraint}{\mathcal{C}}

\newcommand{\R}{\mathbb{R}}

\renewcommand\part[1]{\vspace{.05in}\textbf{(#1)}}

\newcommand{\X}{\mathcal{X}}
\newcommand{\argmax}{\operatornamewithlimits{argmax}}

\newcommand{\RCDF}{\mathrm{Q}}
\newcommand{\RCDFpl}{\mathrm{Q}^+}
\newcommand{\RCDFj}{\RCDF\supj}
\newcommand{\RCDFplj}{\RCDFpl\supj}

\newcommand{\RCDFone}{\RCDF\supa}
\newcommand{\RCDFtwo}{\RCDF\supb}

\newcommand{\dem}{\mathtt{DemParity}}
\newcommand{\dempar}{\dem}

\newcommand{\maxprof}{\mathtt{MaxUtil}}

\newcommand{\eqop}{\mathtt{EqOpt}}
\newcommand{\subto}{\mathrm{s.t.}}

\newcommand{\ratef}{r}
\newcommand{\ratefj}{r_{\distj}}

\newcommand{\constf}{T}
\newcommand{\constfj}{\constf_{\supj,\boldwj}}
\newcommand{\constfa}{\constf_{\supa,\boldwa}}
\newcommand{\constfb}{\constf_{\supb,\boldwb}}

\newcommand{\supa}{_{\popa}}
\newcommand{\supb}{_{\popb}}
\newcommand{\supj}{_{\popj}}

\newcommand{\popa}{\mathsf{A}}
\newcommand{\popb}{\mathsf{B}}
\newcommand{\pops}{\{\popa, \popb\}}

\newcommand{\popj}{\mathsf{j}}

\newcommand{\fraca}{g\supa}
\newcommand{\fracb}{g\supb}
\newcommand{\fracj}{g\supj}

\newcommand{\pol}{\boldsymbol{\tau}}
\newcommand{\polj}{\pol\supj}
\newcommand{\pola}{\pol\supa}
\newcommand{\polb}{\pol\supb}

\newcommand{\mean}{\boldsymbol{\mu}}
\newcommand{\delmean}{\Delta\mean}
\newcommand{\mono}{{\mathrm{thresh}}}

\newcommand{\transfer}{G^{(\popa\to\popb)}}

\newcommand{\pb}{\boldsymbol{\rho}}
\newcommand{\chg}{\boldsymbol{\Delta}}
\newcommand{\dst}{\boldsymbol{\pi}}
\newcommand{\dist}{\boldsymbol{\pi}}

\newcommand{\dista}{\dist\supa}
\newcommand{\distb}{\dist\supb}
\newcommand{\distj}{\dist\supj}

\newcommand{\lamhat}{\widehat{\lambda}}

\newcommand{\util}{\boldsymbol{u}}
\newcommand{\Util}{\mathcal{U}}
\newcommand{\Utila}{\Util\supa}
\newcommand{\Utilb}{\Util\supb}
\newcommand{\Utilj}{\Util\supj}

\newcommand{\hdmd}{\circ}
\newcommand{\maxscore}{C}

\newcommand{\cgood}{c_+}
\newcommand{\cbad}{c_-}

\newcommand{\boldg}{\boldsymbol{g}}

\newcommand{\boldv}{\boldsymbol{v}}

\newcommand{\Monotone}{\mathcal{T}_{\mathrm{thresh}}}

\newcommand{\boldw}{\boldsymbol{w}}
\newcommand{\boldwj}{\boldw\supj}
\newcommand{\boldwa}{\boldw\supa}

\newcommand{\simpi}{\cong_{\dist}}
\newcommand{\simj}{\cong_{\dist\supj}}
\newcommand{\sima}{\cong_{\dist\supa}}
\newcommand{\simb}{\cong_{\dist\supb}}

\newcommand{\Simplex}{\mathrm{Simplex}^{\maxscore-1}}

\newcommand{\boldwb}{\boldw\supb}

\newcommand{\bolde}{\boldsymbol{e}}

\newcommand{\NC}{\mathrm{NC}}

%% file: abstract.tex
Fairness in machine learning has predominantly been studied in static
classification settings without concern for how decisions change the underlying
population over time. Conventional wisdom suggests that fairness criteria
promote the long-term well-being of those groups they aim to protect.

We study how static fairness criteria interact with temporal indicators of
well-being, such as long-term improvement, stagnation, and decline in a variable
of interest. We demonstrate that even in a one-step feedback model, common
fairness criteria in general do not promote improvement over time, and may in
fact cause harm in cases where an unconstrained objective would not.
We completely characterize the delayed impact of three standard criteria,
contrasting the regimes in which these exhibit qualitatively different
behavior. In addition, we find that a natural form of measurement error broadens
the regime in which fairness criteria perform favorably.

Our results highlight the importance of measurement and temporal modeling in the
evaluation of fairness criteria, suggesting a range of new challenges and
trade-offs.

%% file: 01_intro.tex

\section{Introduction}

Machine learning commonly considers static objectives defined on a snapshot of
the population at one instant in time; consequential decisions, in contrast,
reshape the population over time. Lending practices, for example, can shift the
distribution of debt and wealth in the population. Job advertisements 
allocate opportunity. School admissions shape the level of education in
a community.

Existing scholarship on fairness in automated decision-making criticizes
unconstrained machine learning for its potential to \emph{harm} historically
underrepresented or disadvantaged groups in the
population~\citep{whitehouse16,barocasselbst16}. Consequently, a variety of
\emph{fairness criteria} have been proposed as constraints on standard learning
objectives. Even though, in each case, these constraints are clearly intended to
\emph{protect} the disadvantaged group by an appeal to intuition, a rigorous
argument to that effect is often lacking.

In this work, we formally examine under what circumstances fairness criteria do
indeed promote the long-term well-being of disadvantaged groups measured in
terms of a temporal variable of interest. Going beyond the standard
classification setting, we introduce a one-step feedback model of decision-making
that exposes how decisions change the underlying population over time.

Our running example is a hypothetical lending scenario. There are two groups in
the population with features described by a summary statistic, such as a
\emph{credit score}, whose distribution differs between the two groups. The bank can
choose thresholds for each group at which loans are offered. While
group-dependent thresholds may face legal challenges \citep{ross2006}, they are
generally inevitable for some of the criteria we examine. The impact of a
lending decision has multiple facets. A default event not only diminishes profit for
the bank, it also worsens the financial situation of the borrower as reflected
in a subsequent decline in credit score. A successful lending outcome leads to
profit for the bank and also to an increase in credit score for the borrower.

When thinking of one of the two groups as disadvantaged, it makes sense to ask
what lending policies (choices of thresholds) lead to an expected improvement in
the score distribution within that group. An unconstrained bank would maximize
profit, choosing thresholds that meet a break-even point above which it is
profitable to give out loans. One frequently proposed fairness criterion,
sometimes called demographic parity, requires the bank to lend to both groups at
an equal rate. Subject to this requirement the bank would continue to maximize
profit to the extent possible. 
Another criterion, originally called equality of opportunity, equalizes
the \emph{true positive rates} between the two groups, thus requiring the bank to lend in
both groups at an equal rate among individuals who repay their loan. Other
criteria are natural, but for clarity we restrict our attention to these three.

Do these fairness criteria benefit the disadvantaged group? When do they show a
clear advantage over unconstrained classification? Under what circumstances does
profit maximization work in the interest of the individual? These are important
questions that we begin to address in this work.

\subsection{Contributions}
We introduce a one-step feedback model that allows us to quantify the long-term
impact of classification on different groups in the population. We represent
each of the two groups~$\popa$ and~$\popb$ by a \emph{score}
distribution~$\dist\supa$ and~$\dist\supb,$ respectively. The support of these
distributions is a finite set~$\X$ corresponding to the possible values that the
score can assume. We think of the score as highlighting one variable of interest
in a specific domain such that higher score values correspond to a
higher probability of a positive outcome.  
An \emph{institution} chooses selection policies~$\pola,
\polb\colon\X\to[0,1]$ that assign to each value in~$\X$ a number representing
the rate of selection for that value. In our example, these policies specify
the lending rate at a given credit score within a given group. The institution
will always maximize their utility (defined formally later) subject to either
(a) no constraint, or (b) equality of selection rates, or (c) equality of true
positive rates.

We assume the availability of a function~$\chg\colon\X\to\R$ such that $\chg(x)$ provides the
expected change in score for a selected individual at score $x$.  The
central quantity we study is the expected difference in the mean
score in group~$j\in\{\popa, \popb\}$ that results from an institutions policy, $\delmean\supj$ defined formally in Equation~\eqref{eqn:delmean}. 
When modeling the problem, the expected mean difference can also absorb external
factors such as ``reversion to the mean'' so long as they are mean-preserving.
Qualitatively, we distinguish between \emph{long-term improvement}
($\delmean\supj>0$), \emph{stagnation} ($\delmean\supj=0$), and \emph{decline}
($\delmean\supj<0$).
Our findings can be summarized as follows:

\begin{enumerate} 
	\item Both fairness criteria (equal selection rates, equal true
	positive rates) can lead to all possible outcomes (improvement, stagnation, and decline) in natural parameter regimes.  We provide a complete characterization of when each criterion
	leads to each outcome in Section \ref{sec:results}.
	\begin{itemize}
		\item There are a class of settings where equal
		selection rates cause decline, whereas equal true positive rates do not (Corollary \ref{cor:avoid_harm}),
		\item Under a mild assumption, the institution's optimal unconstrained selection policy can never lead to decline (Proposition \ref{prop:mp_noharm}). 
	\end{itemize}

\item We introduce the notion of an \emph{outcome curve} (Figure~\ref{fig:outcome_curve}) which succinctly describes the different regimes in which one criterion is preferable over the others. 

\item We perform experiments on FICO credit score data from 2003 and show that under various models of bank utility and score change, the outcomes of applying fairness criteria are in line with our theoretical predictions.

\item
We discuss how certain types of measurement error (e.g., the bank
underestimating the repayment ability of the disadvantaged group) affect our
comparison. We find that measurement error narrows the regime in which fairness
criteria cause decline, suggesting that measurement should be a factor when
motivating these criteria.
\item We consider alternatives to hard fairness constraints.
\begin{itemize}
	\item We evaluate the optimization problem where fairness criterion is a regularization term in the objective. Qualitatively, this leads to the
	same findings.
	\item We discuss the possibility of optimizing for group score improvement $\delmean\supj$ directly subject to institution utility constraints. The
	resulting solution provides an interesting possible alternative to existing fairness criteria.
\end{itemize}

\end{enumerate}

We focus on the impact of a selection policy over a single epoch. The
motivation is that the designer of a system usually has an understanding of the
time horizon after which the system is evaluated and possibly redesigned.
Formally, nothing prevents us from repeatedly applying our model and tracing
changes over multiple epochs. In reality, however, it is plausible that over greater time
periods, economic background variables might dominate the effect of selection.

Reflecting on our findings, we argue that careful temporal modeling is necessary
in order to accurately evaluate the impact of different fairness criteria on the
population. Moreover, an understanding of measurement error is important in
assessing the advantages of fairness criteria relative to unconstrained
selection. Finally, the nuances of our characterization underline how
intuition may be a poor guide in judging the long-term impact of fairness
constraints.

\subsection{Related work}


Recent work by \citet{hu18shortterm} considers a model for long-term outcomes and fairness in the labor market. They propose imposing the demographic parity constraint in a \emph{temporary} labor market in order to provably achieve an equitable long-term equilibrium in the \emph{permanent} labor market, reminiscent of economic arguments for affirmative action~\citep{foster1992economic}. The equilibrium analysis of the labor market dynamics model allows for specific conclusions relating fairness criteria to long term outcomes. Our general framework is complementary to this type of domain specific approach.

\citet{fuster2017predictably}
consider the problem of fairness in credit markets from a different perspective. Their goal is to study the effect of machine learning on interest rates in different groups at an equilibrium, under a static model without feedback.

\citet{ensign2017runaway} consider feedback loops in predictive
policing, where the police more heavily monitor high crime neighborhoods,
thus further increasing the measured number of crimes in those neighborhoods.
While the work addresses an important temporal phenomenon using the theory of urns, it is rather different from our one-step feedback model both conceptually and technically.

Demographic parity and its related formulations have been considered in numerous
papers~\citep[e.g.][]{Calders2009,zafar17a}. \citet{hardt16equality} introduced the equality of
opportunity constraint that we consider and demonstrated limitations of a broad
class of criteria. \citet{kleinberg16inherent} and~\citet{chould16fair} point out
the tension between ``calibration by group'' and equal true/false positive
rates. These trade-offs carry over to some extent to the case where we only
equalize true positive rates~\citep{weinberger2017calibration}.

A growing literature on fairness in the ``bandits'' setting of learning~\citep[see][\emph{et sequelae}]{roth2016bandits} deals with online decision making that ought not
to be confused with our one-step feedback setting. Finally, there has been much work in the social sciences on analyzing the effect
of affirmative action \citep[see e.g.,][]{keith1985,Kaley2006}.

\subsection{Discussion}


In this paper, we advocate for a view toward long-term outcomes in the
discussion of ``fair'' machine learning. 
We argue that without a careful model of delayed outcomes, we cannot foresee the
impact a fairness criterion would have if enforced as a constraint on a
classification system.  However, if such an accurate outcome model is available,
we show that there are more direct ways to optimize for positive outcomes than
via existing fairness criteria. We outline such an outcome-based solution 
in Section \ref{sec:outcome_based}. Specifically, in the credit setting, the
outcome-based solution corresponds to giving out more loans to the protected
group in a way that reduces profit for the bank compared to unconstrained profit
maximization, but avoids loaning to those who are unlikely to benefit, resulting
in a maximally improved group average credit score. The extent to which such a
solution could form the basis of successful regulation depends on the accuracy
of the available outcome model.

This raises the question if our model of outcomes is rich enough to faithfully
capture realistic phenomena.  By focusing on the impact that selection has on
individuals at a given score, we model the effects for those \emph{not} selected
as zero-mean. For example, not getting a loan in our model has no negative
effect on the credit score of an individual.\footnote{In reality, a denied
credit inquiry may lower one's credit score, but the effect is small compared to
a default event.} This does not mean that wrongful rejection (i.e., a false
negative) has no visible manifestation in our model. If a classifier has a
higher false negative rate in one group than in another, we expect the
classifier to increase the disparity between the two groups (under natural
assumptions). In other words, in our outcome-based model, the harm of denied
opportunity manifests as growing disparity between the groups.  The cost of a
false negative could also be incorporated directly into the outcome-based model by a simple modification (see
Footnote~\ref{footnote:negative_dec_outcomes}). This may be fitting in some
applications where the immediate impact of a false negative to the individual is
not zero-mean, but significantly reduces their future success probability.

In essence, the formalism we propose requires us to understand the two-variable
causal mechanism that translates decisions to outcomes.  This can be seen as
relaxing the requirements compared with recent work 
on avoiding discrimination through causal reasoning that often required stronger
assumptions~\citep{Kusner17,Nabi17,Kilbertus17}. In particular, these works required knowledge of how
sensitive attributes (such as gender, race, or proxies thereof) causally relate
to various other variables in the data. Our model avoids the delicate modeling
step involving the sensitive attribute, and instead focuses on an arguably more
tangible economic mechanism. Nonetheless, depending on the application, such an
understanding might necessitate greater domain knowledge and additional research
into the specifics of the application. This is consistent with much scholarship
that points to the context-sensitive nature of fairness in machine learning.

%% file: 02_problem_setting.tex

\section{Problem Setting}
\label{sec:problem setting }
\label{sec:models_and_notation}
We consider two \emph{groups} $\popa$ and $\popb$, which comprise a $\fraca$ and $\fracb = 1 - \fraca$ fraction of the total population, and an \emph{institution} which makes a binary decision for each individual in each group, called \emph{selection}. 
Individuals in each group are assigned \emph{scores} in $\calX := [\maxscore]$, 
and the scores for group $\popj \in \{\popa,\popb\}$ are distributed according $\dist\supj \in \Simplex$. 
The institution selects a \emph{policy} $\pol:=(\pol\supa,\pol\supb)\in[0,1]^{2\maxscore}$, where $\pol\supj(x)$ corresponds to the probability the institution selects an individual in group $\popj$ with score $x$. 
One should think of a score as an abstract quantity which summarizes how well an individual is suited to being selected; examples are provided at the end of this section.

We assume that the institution is utility-maximizing, but may impose certain constraints to ensure that the policy $\pol$ is \emph{fair}, in a sense described in Section~\ref{sec:fairness_criteria}. We assume that there exists a function $\util: \maxscore \to \R$, such that the institution's expected utility for a policy $\pol$ is given by 
\begin{eqnarray}
\textstyle
\label{eqn:inst_util}
\Util(\pol) = \sum_{\popj \in \pops } g_{\popj}\sum_{x \in \X} \polj(x)\distj(x)\util(x).
\end{eqnarray}
Novel to this work, we focus on the effect of the selection policy $\pol$ on the groups $\popa$ and $\popb$. 
We quantify these \emph{outcomes} in terms of an average effect that a policy $\polj$ has on group $\popj$. Formally, for a function $\chg(x):\X \to \R$, we define the average change of the mean score $\mean\supj$ for group $\popj$ 
\begin{eqnarray}
\textstyle
\label{eqn:delmean}
\delmean\supj (\pol) := \sum_{x \in \X} \dst\supj(x)\pol\supj(x)\chg(x) \:. 
\end{eqnarray}
We remark that many of our results also go through if $\delmean\supj(\pol)$ simply refers to an abstract change in well-being, not necessarily a change in the mean score. Furthermore, it is possible to modify the definition of $\delmean\supj(\pol)$ such that it directly considers outcomes of those who are not selected.\footnote{\label{footnote:negative_dec_outcomes}
If we consider functions $\chg_p(x):\X \to \R$ and $\chg_n(x):\X \to \R$ to represent the average effect of selection and non-selection respectively, then $
\delmean\supj (\pol) := \sum_{x \in \X} \dst\supj(x) \left(\pol\supj(x)\chg_p(x) + (1-\pol\supj(x))\chg_n(x) \right)$.
This model corresponds to replacing $\chg(x)$ in the original outcome definition with $\chg_p(x)-\chg_n(x)$, and adding a offset $\sum_{x \in \X} \dst\supj(x)\chg_n(x)$. Under the assumption that $\chg_p(x)-\chg_n(x)$ increases in $x$, this model gives rise to outcomes curves resembling those in Figure~\ref{fig:outcome_curve} up to vertical translation. All presented results hold unchanged under the further assumption that $\delmean (\accrate^\maxprof) \geq 0$.
}
Lastly, we assume that the \emph{success} of an individual is independent of their group given the score; that is, the score summarizes all relevant information about the success event, so there exists a function $\pb:\X \to [0,1]$ such that individuals of score $x$ succeed with probability $\pb(x)$.

We now introduce the specific domain of credit scores as a running example in the rest of the paper, after which we present two more examples showing the general applicability of our formulation to many domains.

\begin{exmp}[Credit scores] 
	\label{ex:credit_lending} In the setting of loans, scores $x \in [\maxscore]$ represent credit scores, and the bank serves as the institution. The bank chooses to grant or refuse loans to individuals according to a policy $\pol$. Both bank and personal utilities are given as functions of loan repayment, and therefore depend on the success probabilities $\pb(x)$, representing the probability that any individual with credit score $x$ can repay a loan within a fixed time frame. The expected utility to the bank is given by the expected return from a loan, which can be modeled as an affine function of $\pb(x)$: $\util(x) = u_+\pb(x) + u_-(1-\pb(x))$, where $u_+$ denotes the profit when loans are repaid and $u_-$ the loss when they are defaulted on. Individual outcomes of being granted a loan are based on whether or not an individual repays the loan, and a simple model for $\chg(x)$ may also be affine in $\pb(x)$: $\chg(x) = \cgood \pb(x) + \cbad (1 - \pb(x))$, modified accordingly at boundary states. The constant $c_+$ denotes the gain in credit score if loans are repaid and $c_-$ is the score penalty in case of default.
	

\end{exmp}

\begin{exmp}[Advertising]
	\label{ex:advertizing}
	A second illustrative example is given by the case of advertising agencies making decisions about which groups to target. An individual with product interest score $x$ responds positively to an ad with probability $\pb(x)$. The ad agency experiences utility $\util(x)$ related to click-through rates, which increases with $\pb(x)$. Individuals who see the ad but are uninterested may react negatively (becoming less interested in the product), and $\chg(x)$ encodes the interest change. 
	If the product is a  positive good like education or employment opportunities, interest can correspond to well-being.
	Thus the advertising agency's incentives to only show ads to individuals with extremely high interest may leave behind groups whose interest is lower on average. 
	A related historical example occurred in advertisements for computers in the 1980s, where male consumers were targeted over female consumers, arguably contributing to the current gender gap in computing. 
\end{exmp}

\begin{exmp}[College Admissions]
	\label{ex:admissions}
	
	The scenario of college admissions or scholarship allotments can also be considered within our framework. Colleges may select certain applicants for acceptance according to a score $x$, which could be thought encode a ``college preparedness'' measure. The students who are admitted might ``succeed" (this could be interpreted as graduating, graduating with honors, finding a job placement, etc.) with some probability $\pb(x)$ depending on their preparedness. The college might experience a utility $\util(x)$ corresponding to alumni donations, or positive rating when a student succeeds; they might also show a drop in rating or a loss of invested scholarship money when a student is unsuccessful. The student's success in college will affect their later success, which could be modeled generally by $\chg(x)$. In this scenario, it is challenging to ensure that a single summary statistic $x$ captures enough information about a student; it may be more appropriate to consider $x$ as a vector as well as more complex forms of $\pb(x)$. 
\end{exmp}

While a variety of applications are modeled faithfully within our framework, there are limitations to the accuracy with which real-life phenomenon can be measured by strictly binary decisions and success probabilities. Such binary rules are necessary for the definition and execution of existing fairness criteria, (see Sec.~\ref{sec:fairness_criteria}) and as we will see, even modeling these facets of decision making as binary allows for complex and interesting behavior.

\subsection{The Outcome Curve}
\label{sec:pop_outcomes}
We now introduce important outcome regimes, stated in terms of the change in average group score. A policy $(\pola,\polb)$ is said to cause \emph{active harm} to group $\popj$ if $\delmean\supj(\pol\supj) < 0$, \emph{stagnation} if $\delmean\supj(\pol\supj) = 0$, and \emph{improvement} if $\delmean\supj(\pol\supj) > 0$. Under our model, $\maxprof$ policies can be chosen in a standard fashion which applies the same threshold $\pol^\maxprof$ for both groups, and is agnostic to the distributions $\dista$ and $\distb$.  Hence, if we define
\begin{eqnarray}
\delmean\supj^{\maxprof}:= \delmean\supj(\pol^\maxprof)
\end{eqnarray}
we say that a policy causes \emph{relative harm} to group $\popj$ if $\delmean\supj(\polj) < \delmean\supj^{\maxprof}$, and \emph{relative improvement} if $\delmean\supj(\polj) > \delmean\supj^{\maxprof}$. In particular, we focus on these outcomes for a disadvantaged group, and consider whether imposing a fairness constraint improves their outcomes relative to the $\maxprof$ strategy. From this point forward, we take $\popa$ to be disadvantaged or protected group. 

\begin{figure*}
\begin{subfigure}{0.7\textwidth}
\centering
\includegraphics[width=\textwidth]{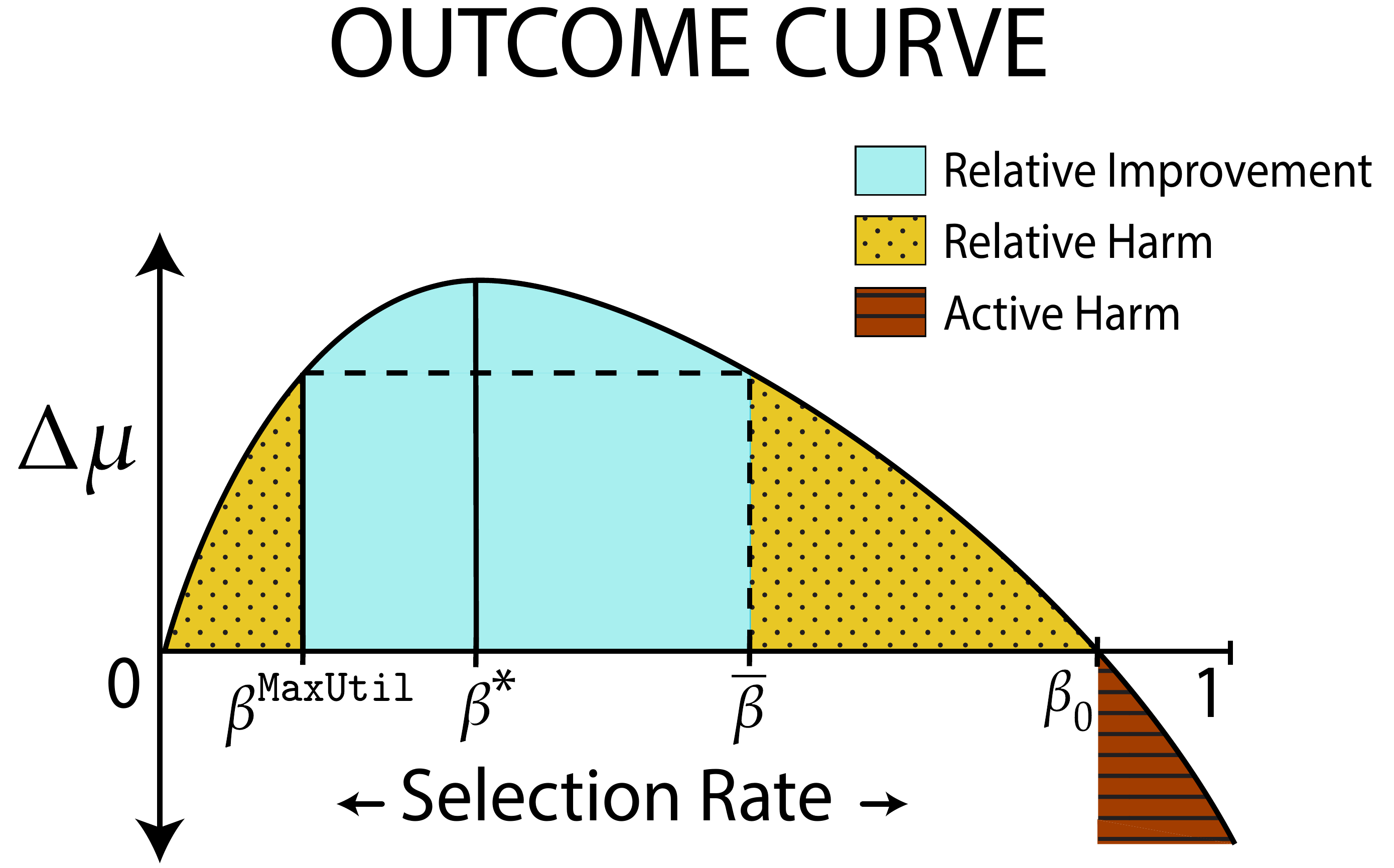}
\caption{}
\label{subfig:outcome_region}
\end{subfigure}%
\begin{subfigure}{.3\textwidth}
\centering
\includegraphics[width=\textwidth]{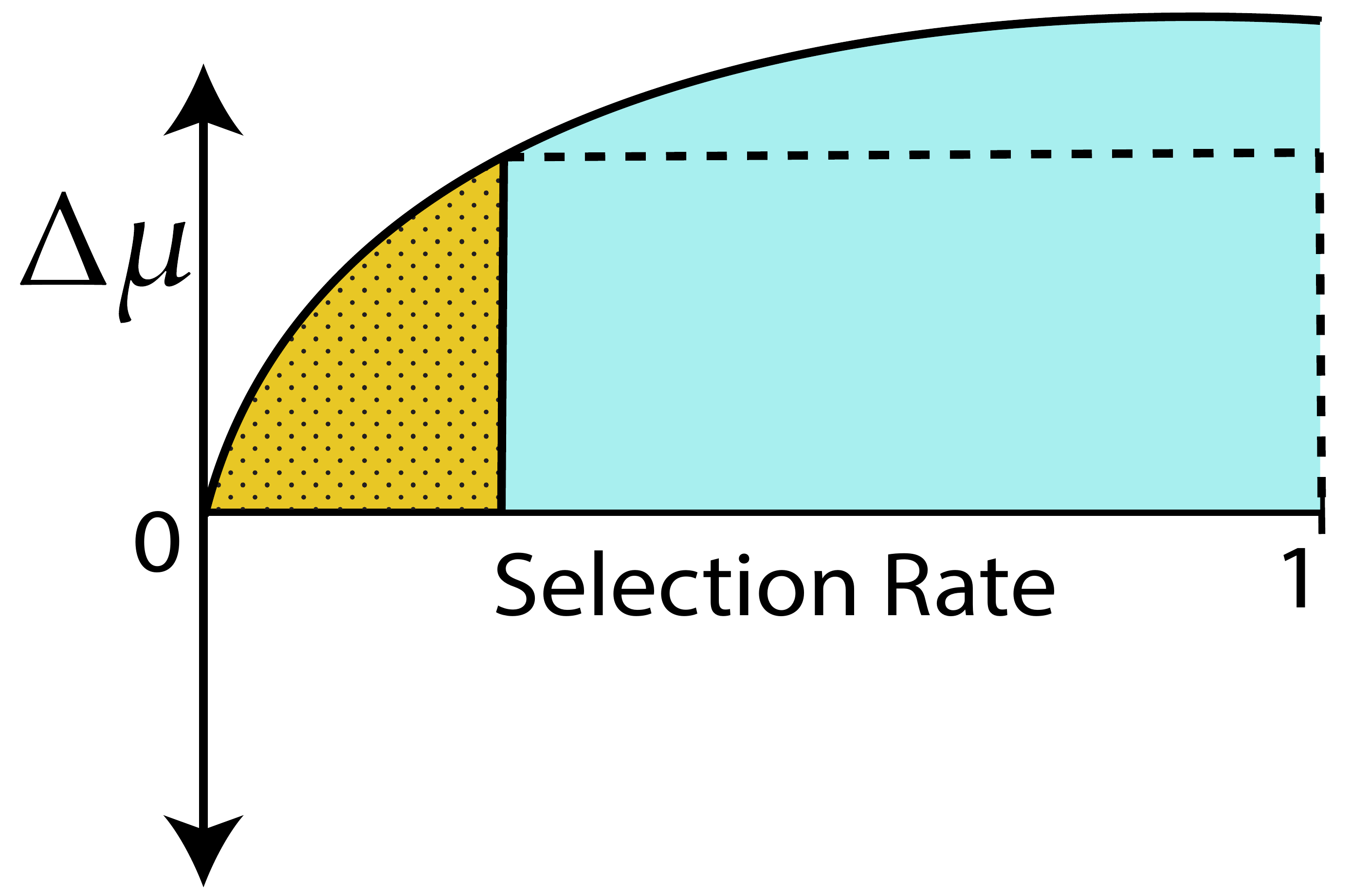}
\caption{}
\label{fig:sub1}
\includegraphics[width=\textwidth]{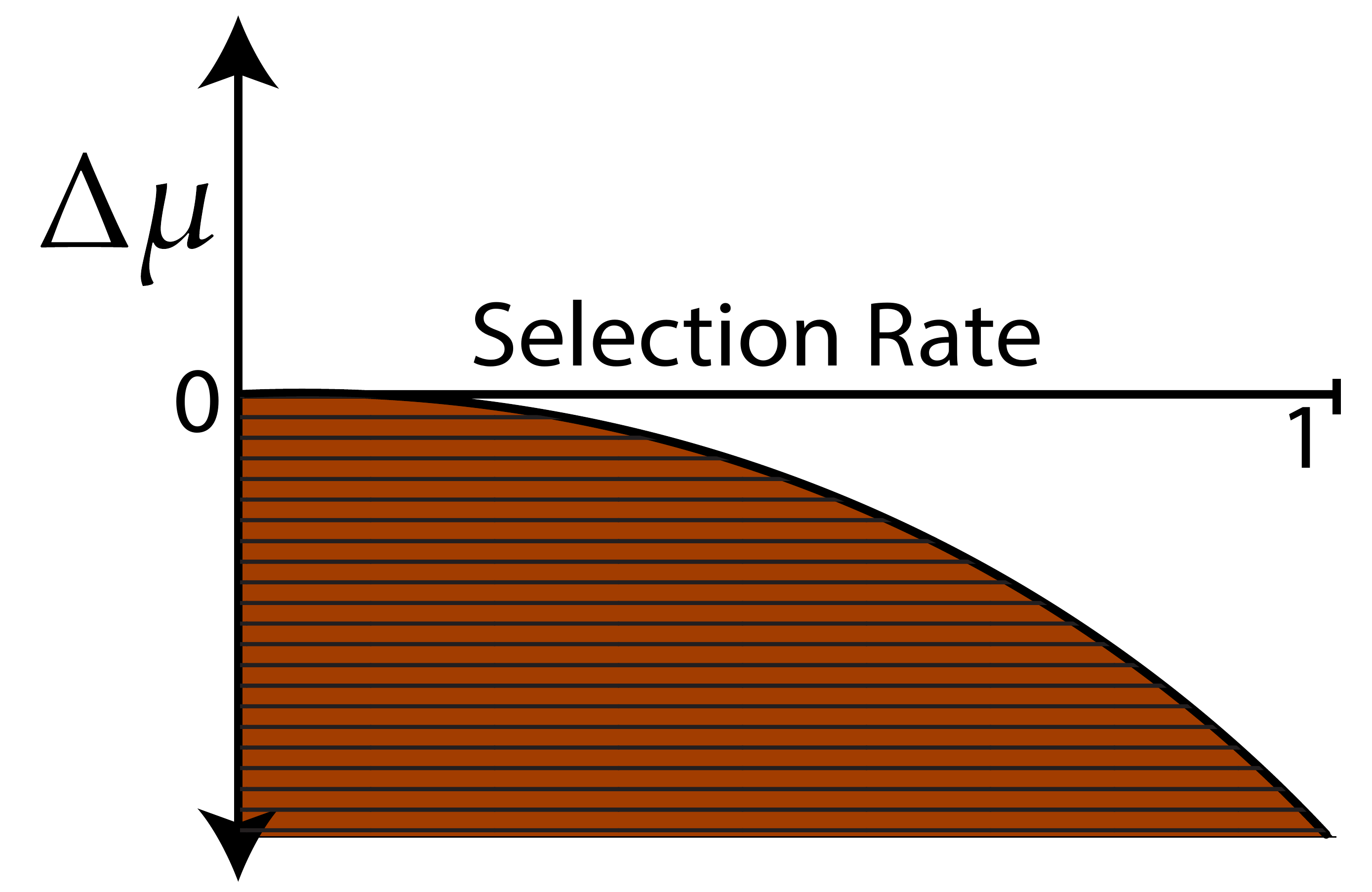}
\caption{}
\label{fig:sub2}
\end{subfigure}
\caption{The above figure shows the \emph{outcome curve}. The horizontal axis represents the selection rate for the population; the vertical axis represents the mean change in score. (a) depicts the full spectrum of outcome regimes, and colors indicate regions of active harm, relative harm, and no harm. In (b): a group that has much potential for gain, in (c): a group that has no potential for gain.}
\label{fig:outcome_curve}
\end{figure*}

Figure~\ref{fig:outcome_curve} displays the important outcome regimes in terms of \emph{selection rates} $\accrate\supj := \sum_{x \in \X} \distj(x)\polj(x)$. This succinct characterization is possible when considering decision rules based on (possibly randomized) score thresholding, in which all individuals with scores above a threshold are selected. In Section~\ref{section:opt_threshold}, we justify the restriction to such \emph{threshold policies} by showing it preserves optimality. In Section~\ref{section:concavity_outcome}, we show that the outcome curve is concave, thus implying that it takes the shape depicted in Figure~\ref{fig:outcome_curve}. To explicitly connect selection rates to decision policies, we define the rate function $\ratef_{\dist}(\polj)$ which returns the proportion of group $\popj$ selected by the policy. We show that this function is invertible for a suitable class of threshold policies, and in fact the outcome curve is precisely the graph of the map from selection rate to outcome $\accrate \mapsto \delmean\supa (\ratef_{\dista}^{-1}(\accrate))$. Next, we define the values of $\accrate$ that mark boundaries of the outcome regions.

\begin{defn}[Selection rates of interest]\label{def:special_betas}
	Given the protected group $\popa$, the following selection rates are of interest in distinguishing between qualitatively different classes of outcomes (Figure \ref{fig:outcome_curve}). We define $\accrate^{\maxprof}$ as the selection rate for $\popa$ under $\maxprof$; $\accrate_0$ as the harm threshold, such that $\delmean\supa(\ratef^{-1}_{\dista}(\accrate_0))~=~0$; $\accrate^*$ as the selection rate such that $\delmean\supa$ is maximized; $\overline{\accrate}$ as the outcome-complement of the $\maxprof$ selection rate, $\delmean\supa \ratef_{\dista}^{-1}(\overline{\accrate}))~=~\delmean\supa (\ratef_{\dista}^{-1}(\accrate^\maxprof))$ with $\overline{\accrate}~>~\accrate^{\maxprof}$.
\end{defn}

\subsection{Decision Rules and Fairness Criteria}
\label{sec:fairness_criteria}
We will consider policies that maximize the institution's total expected utility, potentially subject to a constraint: $\pol \in \Constraint \in [0,1]^{2\maxscore}$ which enforces some notion of ``fairness''. Formally, the institution selects $\tau_* \in \argmax~\Util(\pol) ~\subto~ \pol \in \Constraint$. We consider the three following constraints:
\begin{defn}[Fairness criteria] The \emph{maximum utility} ($\maxprof$) policy corresponds to the null-constraint $\Constraint = [0,1]^{2\maxscore}$, so that the institution is free to focus solely on utility. The \emph{demographic parity} ($\dempar$) policy results in equal selection rates between both groups. Formally, the constraint is
$\Constraint = \left\{(\pola,\polb): \sum_{x \in \X} \dist\supa(x)\pola =
\sum_{x \in \X} \dist\supb(x)\polb\right\}\:.$
The \emph{equal opportunity} ($\eqop$) policy results in equal true positive rates (TPR) between both group, where TPR is defined as $\oppj(\pol) :=\frac{\sum_{x \in \X} \dist\supj(x)\pb(x)\pol(x)}{\sum_{x \in \X} \dist\supj(x)\pb(x)}$.  $\eqop$ ensures that the conditional probability of selection given that the individual will be successful is independent of the population, formally enforced by the constraint
$\Constraint = \left\{(\pola,\polb): \oppa(\pola) = \oppb(\polb) \right\}\,.$
\end{defn}

Just as the expected outcome $\delmean$ can be expressed in terms of selection rate for threshold policies, so can the total utility $\Util$. In the unconstrained cause, $\Util$ varies independently over the selection rates for group $\popa$ and $\popb$; however, in the presence of fairness constraints the selection rate for one group determines the allowable selection rate for the other. The selection rates must be equal for $\dempar$, but for $\eqop$ we can define a \emph{transfer function}, $\transfer$, which for every loan rate $\beta$ in group $\popa$ gives the loan rate in group $\popb$ that has the same true positive rate. Therefore, when considering threshold policies, decision rules amount to maximizing functions of single parameters. This idea is expressed in Figure~\ref{fig:outcome_util_curves}, and underpins the results to follow.

%% file: 03_results.tex

\section{Results}\label{sec:results}

In order to clearly characterize the outcome of applying fairness constraints, we make the following assumption.

\begin{asm}[Institution utilities]\label{asm:institution_util}
	The institution's individual utility function is more stringent than the expected score changes, $\util(x)> 0 \implies \chg(x)>0$. (For the linear form presented in Example~\ref{ex:credit_lending},  $\frac{u_-}{u_+} < \frac{c_-}{c_+}$ is necessary and sufficient.)
\end{asm}

This simplifying assumption quantifies the intuitive notion that institutions take a greater risk by accepting than the individual does by applying. For example, in the credit setting, a bank loses the amount loaned in the case of a default, but makes only interest in case of a payback. Using Assumption~\ref{asm:institution_util}, we can restrict the position of $\maxprof$ on the outcome curve in the following sense.

\begin{prop}[$\maxprof$ does not cause active harm]\label{prop:mp_noharm}
	Under Assumption~\ref{asm:institution_util}, $ 0\leq \delmean^{\maxprof} \leq \delmean^*$.
\end{prop}
We direct the reader to Appendix~\ref{app:main_results_proofs} for the proof of the above proposition, and all subsequent results presented in this section. The results are corollaries to theorems presented in Section~\ref{sec:main_thm_proofs}.



\subsection{Prospects and Pitfalls of Fairness Criteria}
\label{sec:active_harm_fairness}
We begin by characterizing general settings under which fairness criteria act to improve outcomes over unconstrained $\maxprof$ strategies. For this result, we will assume that group $\popa$ is disadvantaged in the sense that the $\maxprof$ acceptance rate for $\popb$ is large compared to relevant acceptance rates for $\popa$.

\begin{cor}[Fairness Criteria can cause Relative Improvement]\label{cor:fairness_rel_imp}
	
	(a) Under the assumption that $\accrate\supa^{\maxprof}<\accratebar$ and $\accrate\supb^{\maxprof} > \accrate^{\maxprof}\supa$,  there exist population proportions $g_0<g_1<1$ such that, for all $\fraca \in [g_0,g_1]$, $\accrate\supa^{\maxprof}<\accrate\supa^{\dempar} < \accratebar$. That is, $\dempar$ causes relative improvement. 
	
	(b) Under the assumption that there exist $\accrate\supa^{\maxprof}<\accrate<\accrate'<\accratebar$ such that $\accrate\supb^{\maxprof} > \transfer(\accrate), \transfer(\accrate')$,
	there exist population proportions $g_2<g_3<1$ such that, for all $\fraca \in [g_2,g_3]$, $\accrate\supa^{\maxprof}<\accrate\supa^{\eqop} <\accratebar$. That is, $\eqop$ causes relative improvement.
\end{cor}



This result gives the conditions under which we can guarantee the existence of settings in which fairness criteria cause improvement relative to $\maxprof$. Relying on machinery proved in Section~\ref{sec:main_thm_proofs}, the result follows from comparing the position of optima on the utility curve to the outcome curve. Figure~\ref{fig:outcome_util_curves} displays a illustrative example of both the outcome curve and the institutions' utility $\Util$ as a function of the selection rates in group~$\popa$. In the utility function~\eqref{eqn:inst_util}, the contributions of each group are weighted by their population proportions $\fracj$, and thus the resulting selection rates are sensitive to these proportions.

As we see in the remainder of this section, fairness criteria can achieve nearly any position along the outcome curve under the right conditions. This fact comes from the potential mismatch between the outcomes, controlled by $\chg$, and the institution's utility $\util$.

\begin{figure}
\centering
\includegraphics[width=.4\textwidth]{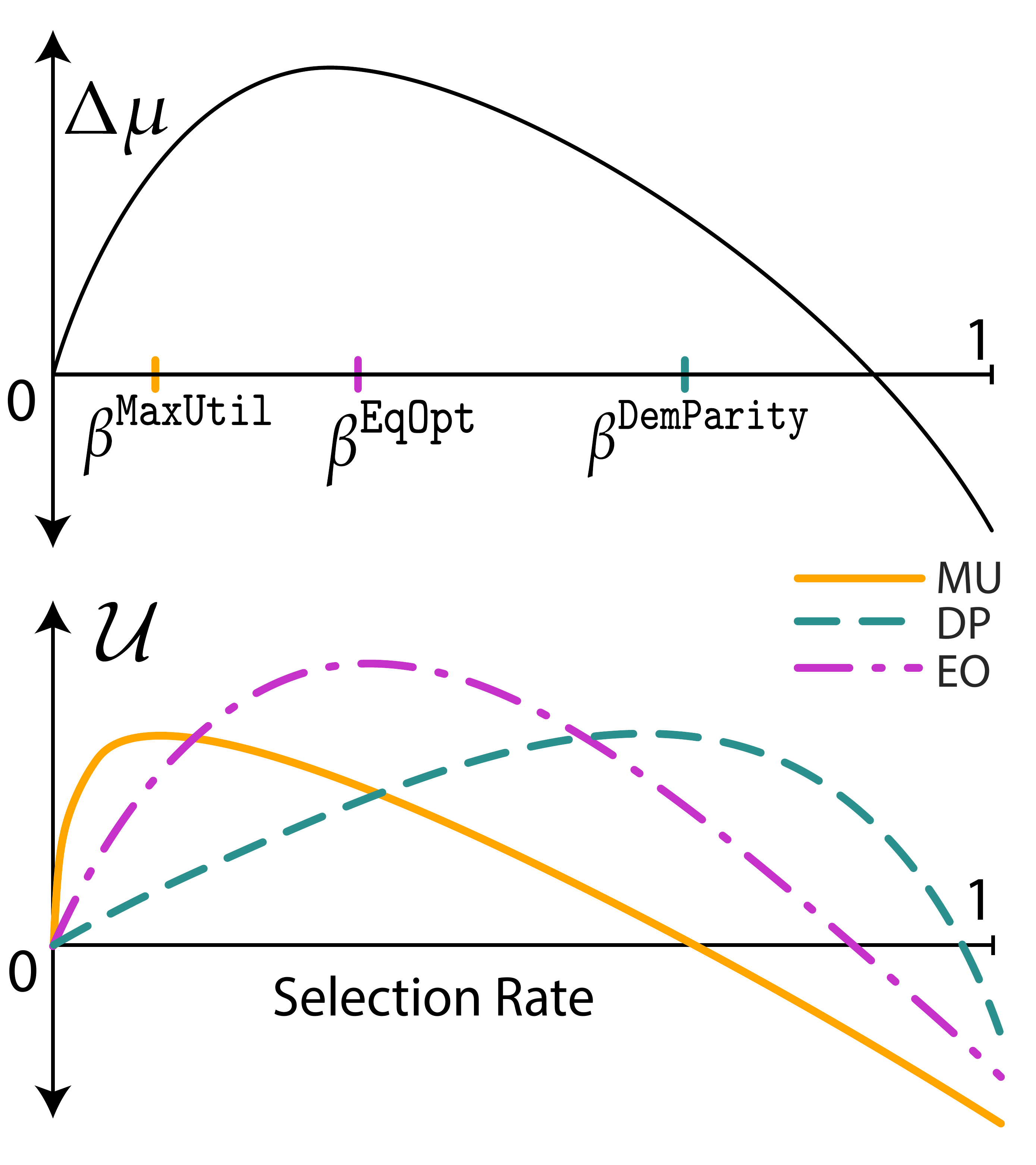}
\caption{Both outcomes $\delmean$ and institution utilities $\Util$ can be plotted as a function of selection rate for one group. The maxima of the utility curves determine the selection rates resulting from various decision rules.}
\label{fig:outcome_util_curves}
\end{figure}

The next theorem implies that $\dempar$ can be bad for long term well-being of the protected group by being over-generous, under the mild assumption that $\delmean\supa(\accrate\supb^{\maxprof})~<~0$:
\begin{cor}[$\dempar$ can cause harm by being over-eager] \label{cor:dp_overeager}  Fix a selection rate $\accrate$. Assume that $\accrate\supb^{\maxprof}~>~\accrate~>~\accrate\supa^{\maxprof}  $. Then, there exists a population proportion $g_0$ such that, for all $\fraca \in [0,g_0]$, $\accrate\supa^{\dempar} > \accrate$. In particular, when $\accrate=\accrate_0$, $\dempar$ causes active harm, and when $\accrate = \accratebar$, $\dempar$ causes relative harm.
\end{cor}
The assumption $\delmean\supa(\accrate\supb^{\maxprof})~<~0$ implies that a policy which selects individuals from group $\popa$ at the selection rate that $\maxprof$ would have used for group $\popb$ necessarily lowers average score in $\popa$. This is one natural notion of protected group $\popa$'s `disadvantage' relative to group $\popb$. In this case, $\dempar$ penalizes the scores of group $\popa$ even more than a naive $\maxprof$ policy, as long as group proportion $\fraca$ is small enough. Again, small $\fraca$ is another notion of group disadvantage. 

Using credit scores as an example, Corollary~\ref{cor:dp_overeager} tells us that an overly aggressive fairness criterion will give too many loans to people in a protected group who cannot pay them back, hurting the group's credit scores on average. In the following theorem, we show that an analogous result holds for $\eqop$.

\begin{cor}[$\eqop$ can cause harm by being over-eager] \label{cor:eqop_overeager} Suppose that $\accrate\supb^{\maxprof}~>~\transfer(\accrate)$ and~$\accrate~>~\accrate\supa^{\maxprof}$. Then, there exists a population proportion $g_0$ such that, for all $\fraca  \in [0,g_0]$, $\accrate^{\eqop}_{\popa}~>~\accrate$. In particular, when $\accrate~=~\accrate_0$, $\eqop$ causes active harm, and when $\accrate~=~\accratebar$, $\eqop$ causes relative harm.
\end{cor} 

We remark that in Corollary~\ref{cor:eqop_overeager}, we rely on the \emph{transfer function}, $\transfer$, which for every loan rate $\beta$ in group $\popa$ gives the loan rate in group $\popb$ that has the same true positive rate. Notice that if $\transfer$ were the identity function, Corollary~\ref{cor:dp_overeager} and Corollary~\ref{cor:eqop_overeager} would be exactly the same. Indeed, our framework (detailed in Section~\ref{sec:main_thm_proofs} and Appendix~\ref{sec:char_fairness_solns}) unifies the analyses for a large class of fairness constraints that includes $\dempar$ and $\eqop$ as specific cases, and allows us to derive results about impact on $\delmean$ using general techniques. 
In the next section, we present further results that compare the fairness criteria, demonstrating the usefulness of our technical framework.

\subsection{Comparing $\eqop$ and $\dempar$}

Our analysis of the acceptance rates of $\eqop$ and $\dempar$ in Section~\ref{sec:main_thm_proofs} suggests that it is difficult to compare $\dempar$ and $\eqop$ without knowing the full distributions $\dist\supa, \dist\supb$, which is necessary to compute the transfer function $\transfer$. In fact, we have found that settings exist both in which $\dempar$ causes harm while $\eqop$ causes improvement and in which $\dempar$ causes improvement while $\eqop$ causes harm. There cannot be one general rule as to which fairness criteria provides better outcomes in all settings. We now present simple sufficient conditions on the geometry of the distributions for which $\eqop$ is always better than $\dempar$ in terms of $\delmean\supa$.

\begin{cor}[$\eqop$ may avoid active harm where $\dempar$ fails]\label{cor:avoid_harm} Fix a selection rate $\accrate$. Suppose $\dist\supa, \dist\supb$ are identical up to a translation with $\mean\supa~<~\mean\supb$, i.e. $\dist\supa(x)~=~\dist\supb(x+(\mean\supb-\mean\supa))$. For simplicity, take $\pb(x)$ to be linear in $x$. Suppose \[ \accrate > \sum_{x > \mean \supa } \dist\supa. \] 
Then there exists an interval $[g_1, g_2]~\subseteq~[0,1]$, such that $\forall \fraca~>~g_1$, $\accrate^{\eqop}~<~\accrate$ while $\forall \fraca~<~g_2$, $\accrate^{\dempar}~>~\accrate$. In particular, when $\accrate~=~\accrate_0$, this implies $\dempar$ causes active harm but $\eqop$ causes improvement for $\fraca ~\in~[g_1, g_2]$, but for any $\fraca $ such that $\dempar$ causes improvement, $\eqop$ also causes improvement.
\end{cor} 

To interpret the conditions under which Corollary $\ref{cor:avoid_harm}$ holds, consider when we might have $\accrate_0 > \sum_{x > \mean \supa } \dist\supa$. This is precisely when $\delmean\supa( \sum_{x > \mean \supa } \dist\supa) > 0$, that is, $\delmean\supa > 0$ for a policy that selects every individual whose score is above the group $\popa$ mean, which is reasonable in reality. Indeed, the converse would imply that group $\popa$ has such low scores that even selecting all above average individuals in $\popa$ would hurt the average score. In such a case, Corollary~\ref{cor:avoid_harm} suggests that $\eqop$ is better than $\dempar$ at avoiding active harm, because it is more conservative. A natural question then is: can $\eqop$ cause relative harm by being too stingy? 

\begin{cor}[$\dempar$ never loans less than $\maxprof$, but $\eqop$ might] \label{cor:eqop_underloan} Recall the definition of the TPR functions $\tpr\supj$, and suppose that the $\maxprof$ policy $\pol^{\maxprof}$ is such that
\begin{eqnarray}
\accrate^\maxprof\supa < \accrate^\maxprof\supb\text{ and }
\oppa(\pol^\maxprof) > \oppb(\pol^\maxprof)
\end{eqnarray}
Then $\accrate\supa^\eqop~<~\accrate\supa^\maxprof~<~\accrate\supa^\dempar$. That is, $\eqop$ causes relative harm by selecting at a rate lower than $\maxprof$.
\end{cor}

The above theorem shows that $\dempar$ is never stingier than $\maxprof$ to the protected group $\popa$, as long as a $\popa$ is disadvantaged in the sense that $\maxprof$ selects a larger proportion of $\popb$ than $\popa$. On the other hand, $\eqop$ can select less of group $\popa$ than $\maxprof$, and by definition, cause relative harm. This is a surprising result about $\eqop$, and this phenomenon arises from high levels of in-group inequality for group $\popa$. Moreover, we show in Appendix \ref{app:main_results_proofs} that there are parameter settings where the conditions in Corollary \ref{cor:eqop_underloan} are satisfied even under a stringent notion of disadvantage we call CDF domination, described therein.


\section{Relaxations of Constrained Fairness}
\subsection{Regularized fairness} 
In many cases, it may be unrealistic for an institution to ensure that fairness constraints are met exactly. However, one can consider ``soft'' formulations of fairness constraints which either penalized the differences in acceptance rate ($\dempar$) or the differences in TPR ($\eqop$). In Appendix~\ref{sec:char_fairness_solns}, we formulate these soft constraints as regularized objectives. For example, a soft-$\dempar$ can be rendered as
\begin{eqnarray}
\max_{\pol := \pola,\polb} \Util(\pol) - \lambda \Phi(\langle \dista,\pola \rangle -\langle \distb,\polb \rangle )~,
\end{eqnarray}
where $\lambda > 0$ is a regularization parameter, and $\Phi(t)$ is a convex regularization function. We show that the solutions to these objectives are threshold policies, and can be fully characterized in terms of the group-wise selection rate. We also make rigorous the notion that policies which solve the soft-constraint objective interpolate between $\maxprof$ policies at $\lambda = 0$ and hard-constrained policies ($\dempar$ or $\eqop$) as $\lambda \to \infty$. This fact is clearly demonstrated by the form of the solutions in the special case of the regularization function $\Phi(t) = |t|$, provided in the appendix.

\subsection{Fairness Under Measurement Error} 
Next, consider the implications of an institution with imperfect knowledge of scores. Under a simple model in which the estimate of an individual's score $X\sim\dst$ is prone to errors $e(X)$ such that $X~+~e(X):=\widehat X~\sim~\widehat\dst$. Constraining the error to be negative results in the setting that scores are systematically \emph{underestimated}. In this setting, it is equivalent to consider the CDF of underestimated distribution $\widehat\dist$ to be \emph{dominated} by the CDF true distribution $\dst$, that is $\sum_{x \ge c} \widehat\dist(x) \le \sum_{x \ge c} \dist(x)$ for all $c \in [\maxscore]$. Then we can compare the institution's behavior under this estimation to its behavior under the truth.

\begin{prop}[Underestimation causes underselection]\label{prop:underestim}
Fix the distribution of $\popb$ as $\distb$ and let $\accrate$ be the acceptance rate of $\popa$ when the institution makes the decision using perfect knowledge of the distribution $\dista$. Denote $\widehat\accrate$ as the acceptance rate when the group is instead taken as $\widehat\dst\supa$. Then $\accrate\supa^\maxprof~>~\widehat\accrate\supa^\maxprof$ and $\accrate\supa^\dempar>\widehat\accrate\supa^\dempar$. If the errors are further such that the true TPR dominates the estimated TPR, it is also true that $\accrate\supa^\eqop~>~\widehat\accrate\supa^\eqop$.
\end{prop}

Because fairness criteria encourage a higher selection rate for disadvantaged groups (Corollary~\ref{cor:fairness_rel_imp}), systematic underestimation widens the regime of their applicability. Furthermore, since the estimated $\maxprof$ policy underloans, the region for relative improvement in the outcome curve (Figure~\ref{fig:outcome_curve}) is larger, corresponding to more regimes under which fairness criteria can yield favorable outcomes. Thus the potential for measurement error should be a factor when motivating these criteria.

\subsection{Outcome-based alternative} \label{sec:outcome_based}
As explained in the preceding sections, fairness criteria may actively harm disadvantaged groups. It is thus natural to consider a modified decision rule which involves the explicit maximization of $\delmean\supa$. In this case, imagine that the institution's primary goal is to aid the disadvantaged group, subject to a limited profit loss compared to the maximum possible expected profit $\Util^\maxprof$. The corresponding problem is as follows.
\begin{align}
\max_{\pola} \delmean_\popa(\pola)~~\text{s.t.}~~\Util\supa^{\maxprof} - \Util(\pol) < \delta\:.
\end{align}
Unlike the fairness constrained objective, this objective no longer depends on group $\popb$ and instead depends on our model of the mean score change in group $\popa$, $\delmean_\popa$.
\begin{prop}[Outcome-based solution]\label{prop:outcome}
In the above setting, the optimal bank policy $\pol_\popa$ is a threshold policy with selection rate $\accrate = \min\{\accrate^*, \accrate^{\max}\}$, where $\accrate^*$ is the outcome-optimal loan rate and $\accrate^{\max}$ is the maximum loan rate under the bank's ``budget''. 
\end{prop}

The above formulation's advantage over fairness constraints is that it directly optimizes the outcome of $\popa$ and can be approximately implemented given reasonable ability to predict outcomes. Importantly, this objective shifts the focus to outcome modeling, highlighting the importance of domain specific knowledge. Future work can consider strategies that are robust to outcome model errors.

%% file: 06_main_technical_material.tex
\section{Optimality of Threshold Policies}\label{section:opt_threshold}
Next, we move towards statements of the main theorems underlying the results presented in Section~\ref{sec:results}.
	We begin by establishing notation which we shall use throughout. Recall that $\hdmd$ denotes the Hadamard product between vectors.  We identify functions mapping $\X \to \R$ with vectors in $\R^{\maxscore}$. We also define the group-wise utilities 
	\begin{eqnarray}
	\Util\supj(\polj) := \sum_{x \in \X}\distj(x) \polj(x)\util(x)~,
	\end{eqnarray}
	so that for $\pol = (\pola,\polb)$, $\Util(\pol):= \fraca \Util\supa(\pola) + \fracb\Util\supb(\polb)$.

 	First, we formally describe threshold policies, and rigorously justify why we may always assume without loss of generality that the institution adopts policies of this form.
	\begin{defn}[Threshold selection policy] \label{def:monotone_pols}
		A single group selection policy $\pol \in [0,1]^\maxscore$ is called a \emph{threshold policy} if it has the form of a randomized threshold on score:
		\begin{align} \pol_{c,\gamma}= \begin{cases}
		1, & x > c \\
		\gamma, & x=c \\
		0, & x < c
		\end{cases} ~, \text{ for some }~c \in [\maxscore]~\text{and}~ \gamma \in (0,1] \:.
		\end{align} 
	\end{defn}

	As a technicality, if no members of a population have a given score $x \in \X$, there may be multiple threshold policies which yield equivalent selection rates for a given population. 
	To avoid redundancy, we introduce the notation $\polj \simj \polj'$ to mean that the set of scores on which $\polj$ and $\polj'$ differ has probability $0$ under $\distj$; 
	formally, $\sum_{x:\polj(x) \ne \polj(x)} \distj(x) = 0$. 
	For any distribution $\distj$, $\simj$ is an equivalence relation. 
	Moreover, we see that if $\polj\simj \polj'$, 
	then $\polj$ and $\polj'$ both provide the same utility for the institution, 
	induce the same outcomes for individuals in group $\popj$, and
	have the same selection and true positive rates. 
	Hence, if $(\pola,\polb)$ is an optimal solution to any of $\maxprof$, $\eqop$, or $\dempar$, so is any $(\pola',\polb')$ for which $\pola \sima \pola'$ and $\polb \simb \polb'$.  

	For threshold policies in particular, their equivalence class under $\simj$ is uniquely determined by the selection rate function,
		\begin{eqnarray}
		\ratef_{\distj}(\polj) := \sum_{x \in \X} \distj(x) \polj(x)~,
		\end{eqnarray}
		which denotes the fraction of group $\popj$ which is selected. Indeed, we have the following lemma (proved in Appendix~\ref{sec:lem_pol_sim}): 

	\begin{lem}\label{lem:pol_sim_lemma} Let $\polj$ and $\polj'$ be threshold policies. Then $\polj \simj \polj'$ if and only if $\ratef_{\distj}(\polj) = \ratef_{\distj}(\polj')$. 
	Further, $\ratef_{\distj}(\polj)$ is a bijection from $\Monotone(\distj)$ to $[0,1]$, where $\Monotone(\distj)$ is the set of equivalence classes between threshold policies under $\simj$. Finally, $\distj \circ \ratef_{\distj}^{-1}(\accratej)$ is well defined. 
	\end{lem}
	Remark that $\ratef_{\distj}^{-1}(\accratej)$ is an equivalence class rather than a single policy. 
	However, $\distj \circ \ratef_{\distj}^{-1}(\polj)$  is well defined, meaning that $\distj \circ \polj = \distj \circ \polj'$ for any two policies in the same equivalence class. 
	Since all quantities of interest will only depend on policies $\polj$ through $\distj \circ \polj$,  it does not matter \emph{which} representative of $\ratef_{\distj}^{-1}(\accratej)$ we pick.
	Hence, abusing notation slightly, we shall represent $\Monotone(\distj)$ by choosing one representative from each equivalence class under $\simj$\footnote{One  way to do this is to consider the set of all threshold policies $\pol_{c,\gamma}$ such that, $\gamma = 1$ if $\distj(c) = 0$ and $\distj(c-1) > 0$ if $\gamma = 1$ and $c > 1$.}. 

	It turns out the policies which arise in this away are always optimal in the sense that, for a given loan rate $\beta_j$, 
	the threshold policy $\ratef_{\distj}^{-1}(\beta_j)$ is the (essentially unique) policy which maximizes both the institution's utility and the utility of the group. Defining the group-wise utility, 
	\begin{eqnarray} \label{eq:util_group_shorthand}
			\Util\supj(\polj) := \sum_{x \in \X} \util(x) \distj(x) \polj(x)~,
	\end{eqnarray}
	we have the following result: 
	\begin{prop}[Threshold policies are preferable]\label{prop:monotone_best} Suppose that $\util(x)$ and $\chg(x)$ are strictly increasing in $x$. Given any loaning policy $\polj$ for population with distribution $\dist_j$, then the policy $\pol\supj^{\mono} := \ratef_{\distj}^{-1}(\ratef_{\distj}(\polj)) \in \Monotone(\distj)$ satisfies 
	\begin{eqnarray}
	\delmean\supj(\polj^\mono) \ge \delmean\supj(\polj)~\text{and}~\Util\supj(\polj^\mono) \ge \Util\supj(\polj)~.
	\end{eqnarray}
	 Moreover, both inequalities hold with equality if and only if $\pol\supj \simj \pol\supj^{\mono}$. 
	\end{prop} 
	The map $\polj \mapsto \ratefj^{-1}(\ratefj(\polj))$ can be thought of transforming an arbitrary policy $\polj$ into a threshold policy with the same selection rate. 
	In this language, the above proposition states that this map never reduces institution utility or individual outcomes. 
	We can also show that optimal $\maxprof$ and $\dempar$ policies are threshold policies, as well as all $\eqop$ policies under an additional assumption:
	\begin{prop}[Existance of optimal threshold policies under fairness constraints]\label{prop:opt_threshold_fair} Suppose that $\util(x)$ is strictly increasing in $x$. Then all optimal $\maxprof$ policies $(\pola,\polb)$ 
	satisfy $\polj \simj \ratefj^{-1}\left(\ratefj(\polj)\right)$  for $\popj \in \pops$.
	The same holds for all optimal $\dempar$ policies, and if in addition $\util(x)/\pb(x)$ is increasing, the same is true for all optimal $\eqop$ policies. 
	\end{prop}

	To prove proposition \ref{prop:monotone_best}, we invoke the following general lemma which is proved using standard convex analysis arguments (in Appendix~\ref{sec:optimal_best_proof}):
	\begin{lem}\label{lem:optimal_best}
		Let $\boldv \in \R^{\maxscore}$, and let $\boldw \in \R_{>0}^{\maxscore}$, 
		and suppose either that $\boldv(x)$  is increasing in $x$, and $\boldv(x)/\boldw(x)$ is increasing or, $\forall x \in \X,~\boldw(x) = 0$. Let $\dist \in \Simplex$ and fix $t \in [0,\sum_{x \in \X} \dist(x) \cdot \boldw(x)]$. Then any
		\begin{eqnarray}\label{eq:optim}
		\pol^* \in \arg\max_{\pol \in [0,1]^C}\langle \boldv \circ \dist, \pol \rangle \quad \subto \quad \langle \dist \circ \boldw,\pol \rangle = t
		\end{eqnarray}
		satisfies $\pol^* \simpi \ratef_{\dist}^{-1}(\ratef_{\dist}(\pol^*))$. Moreover, at least one maximizer $\pol^* \in \Monotone(\dist)$ exists. 
	\end{lem}
	\begin{proof} [Proof of Proposition~\ref{prop:monotone_best}] 
	We will first prove Proposition~\ref{prop:monotone_best} for the function $\Util\supj$.
	Given our nominal policy $\polj$, let $\accrate_j = \ratef_{\distj}(\polj)$. 
	We now apply Lemma~\ref{lem:optimal_best} with $\boldv(x) = \util(x)$ and $\boldw(x) = 1$. 
	For this choice of $\boldv$ and $\boldw$, $\langle \boldv, \pol \rangle = \Util\supj(\pol)$ and that $\langle \dist_j \circ \boldw,\pol = \ratef_{\distj}(\pol)$. 
	Then, if $\polj \in \arg\max_{\pol}\Util\supj(\pol) ~\subto~\ratef_{\distj}(\pol) = \accrate\supj$, 
	Lemma~\ref{eq:optim} implies that $\polj \simj \ratefj^{-1}(\ratefj(\polj))$. 

	On the other hand, assume that $\polj \simj \ratefj^{-1}\left(\ratefj(\polj)\right)$. 
	We show that $\ratefj^{-1}(\ratefj(\polj))$ is a maximizer; 
	which will imply that $\polj$ is a maximizer since $\polj \simj \ratefj^{-1}(\ratefj(\polj))$ 
	implies that $\Utilj(\polj) = \polj \simj \ratefj^{-1}(\ratefj(\polj))$. 
	By Lemma~\ref{lem:optimal_best} there exists a maximizer $\polj^* \in \Monotone(\dist)$, 
	which means that $\polj^* = \ratefj^{-1}(\ratefj(\polj^*))$. 
	Since  $\polj^*$ is feasible, we must have $\ratefj(\polj^*) = \ratefj(\polj)$, 
	and thus $\polj^* = \ratefj^{-1}(\ratefj(\polj))$, as needed. The same argument follows verbatim if we instead choose $\boldv(x) = \chg(x)$, and compute $\langle \boldv, \pol \rangle = \delmean\supj(\pol)$.
	\end{proof}
	We now argue Proposition~\ref{prop:opt_threshold_fair} for $\maxprof$, as it is a straightforward application of Lemma~\ref{lem:optimal_best}. We will prove Proposition~\ref{prop:opt_threshold_fair} for $\dempar$ and $\eqop$ separately in Sections~\ref{sec:dem_par_proof} and~\ref{sec:eqop_proof}.
	
		\begin{proof}[Proof of Proposition~\ref{prop:opt_threshold_fair} for $\maxprof$] $\maxprof$ follows from  lemma \ref{lem:optimal_best} with $\boldv(x) = \util(x)$, and $t = 0$ and $\boldw = \mathbf{0}$. 
			
		\end{proof}
		

\subsection{Quantiles and Concavity of the Outcome Curve\label{section:concavity_outcome}}
	To further our analysis, we now introduce left and right quantile functions, allowing us to specify thresholds in terms of both selection rate and score  cutoffs. 

	\begin{defn}[Upper quantile function] Define $\RCDF$ to be the upper quantile function corresponding to $\dist$, i.e. \begin{align}
		\RCDFj(\accrate) = \argmax\{c:\sum_{x=c}^C\distj(x) > \accrate\}
		\quad\text{and}\quad \RCDFplj(\accrate):= \argmax\{c:\sum_{x=c}^C\distj(x) \ge \accrate\}\:.
		\end{align}
	\end{defn}
	Crucially $\RCDF(\accrate)$ is continuous from the right, and $\RCDFpl(\accrate)$ is continuous from the left.
	Further, $\RCDF(\cdot)$ and $\RCDFpl(\cdot)$ allow us to compute derivatives of key functions, like the mapping from selection rate $\accrate$ to the group outcome associated with a policy of that rate, $\delmean (\ratef_{\pi}^{-1}(\accrate))$. Because we take $\dist$ to have discrete support, all functions in this work are \emph{piecewise linear}, so we shall need to distinguish between the left and right derivatives, defined as follows
	\begin{eqnarray}
	\partial_- f(x):= \lim_{t \to 0^-}\frac{f(x + t) - f(x)}{t} \quad \text{and} \quad \partial_+ f(y):= \lim_{t \to 0^+}\frac{f(y + t) - f(y)}{t} \:.
	\end{eqnarray}
	For $f$ supported on $[a,b]$, we say that $f$ is left- (resp. right-) differentiable if $\partial_-f(x)$ exists for all $x \in (a,b]$ (resp. $\partial_+ f(y)$ exists for all $y \in [a,b)$). We now state the fundamental derivative computation which underpins the results to follow: 

	\begin{lem}\label{lem:der_comp} Let $\bolde_{x}$ denote the vector such that $\bolde_x(x) = 1$, and $\bolde_x(x') = 0$ for $x' \ne x$. Then $\distj \circ \ratef_{\distj}^{-1}(\accrate): [0,1] \to [0,1]^{\maxscore}$ is continuous, and has left and right derivatives 
	\begin{eqnarray}
	\partial_+ \left(\distj \circ \ratef_{\distj}^{-1}(\accrate)\right) = \bolde_{\RCDF(\beta)} 
	\quad \text{and}\quad \partial_- \left(\distj \circ \ratef_{\distj}^{-1}(\accrate) \right)= \bolde_{\RCDFpl(\beta)} \:.
	\end{eqnarray}
	\end{lem}
	The above lemma is proved in Appendix~\ref{sec:der_comp_proof}. Moreover, Lemma~\ref{lem:der_comp} implies that the outcome curve is concave under the assumption that $\chg(x)$ is monotone:
	\begin{prop}\label{prop:beta_concavity} Let $\dist$ be a distribution over $C$ states. Then $\accrate \mapsto \delmean(\ratef_{\dist}^{-1}(\accrate))$ is concave. In fact, if $\boldw(x)$ is any non-decreasing map from $\X \to \R$, $\accrate \mapsto \langle \boldw, \ratef_{\dist}^{-1}(\accrate) \rangle$ is concave.
	\end{prop}
		\begin{proof} 
		Recall that a univariate function $f$ is concave (and finite) on $[a,b]$ if and only (a) $f$ is left- and right-differentiable, (b) for all $x \in (a,b)$, $\partial_-f(x) \ge \partial_+f(x)$ and (c) for any $x > y$, $\partial_- f(x) \le \partial_+ f(y)$.  

		Observe that $\delmean(\ratef_{\dist}^{-1}(\accrate)) = \langle \chg, \dist \circ  \ratef_{\dist}^{-1}(\accrate) \rangle $. By Lemma~\ref{lem:der_comp}, $\dist \circ  \ratef_{\dist}^{-1}(\accrate)$ has right and left derivatives $\bolde_{\RCDF(\accrate)}$ and $\bolde_{\RCDFpl(\accrate)}$. Hence, we have that
		\begin{eqnarray}
		\partial_+ \delmean(\accrateb) = \chg(\RCDF(\accrateb))\quad\text{and}\quad \partial_- \delmean(\accrateb) = \chg(\RCDFpl(\accrateb)) \:.
		\end{eqnarray}
		Using the fact that $\chg(x)$ is monotone, and that $\RCDF \le \RCDFpl$, we see that $\partial_+ \delmean(f_{\dist}^{-1}(\accrateb)) \le \partial_- \delmean(f_{\dist}^{-1}(\accrateb))  $, and that $\partial \delmean(f_{\dist}^{-1}(\accrateb))$ and $\partial_+ \delmean(f_{\dist}^{-1}(\accrateb))$ are non-increasing, from which it follows that $\delmean(f_{\dist}^{-1}(\accrateb))$ is concave. The general concavity result holds by replacing $\chg(x)$ with $\boldw(x)$.
		\end{proof}

		%
	%

%% file: 07_proofs_of_theorems.tex

\section{Proofs of Main Theorems}  \label{sec:main_thm_proofs}

We are now ready to present and prove theorems that characterize the selection rates under fairness constraints, namely $\dempar$ and $\eqop$. These characterizations are crucial for proving the results in Section \ref{sec:results}. Our computations also generalize readily to other linear constraints, in a way that will become clear in Section \ref{sec:eqop_proof}.

\begin{figure*}[t]

\centering
\includegraphics[width=.4\textwidth]{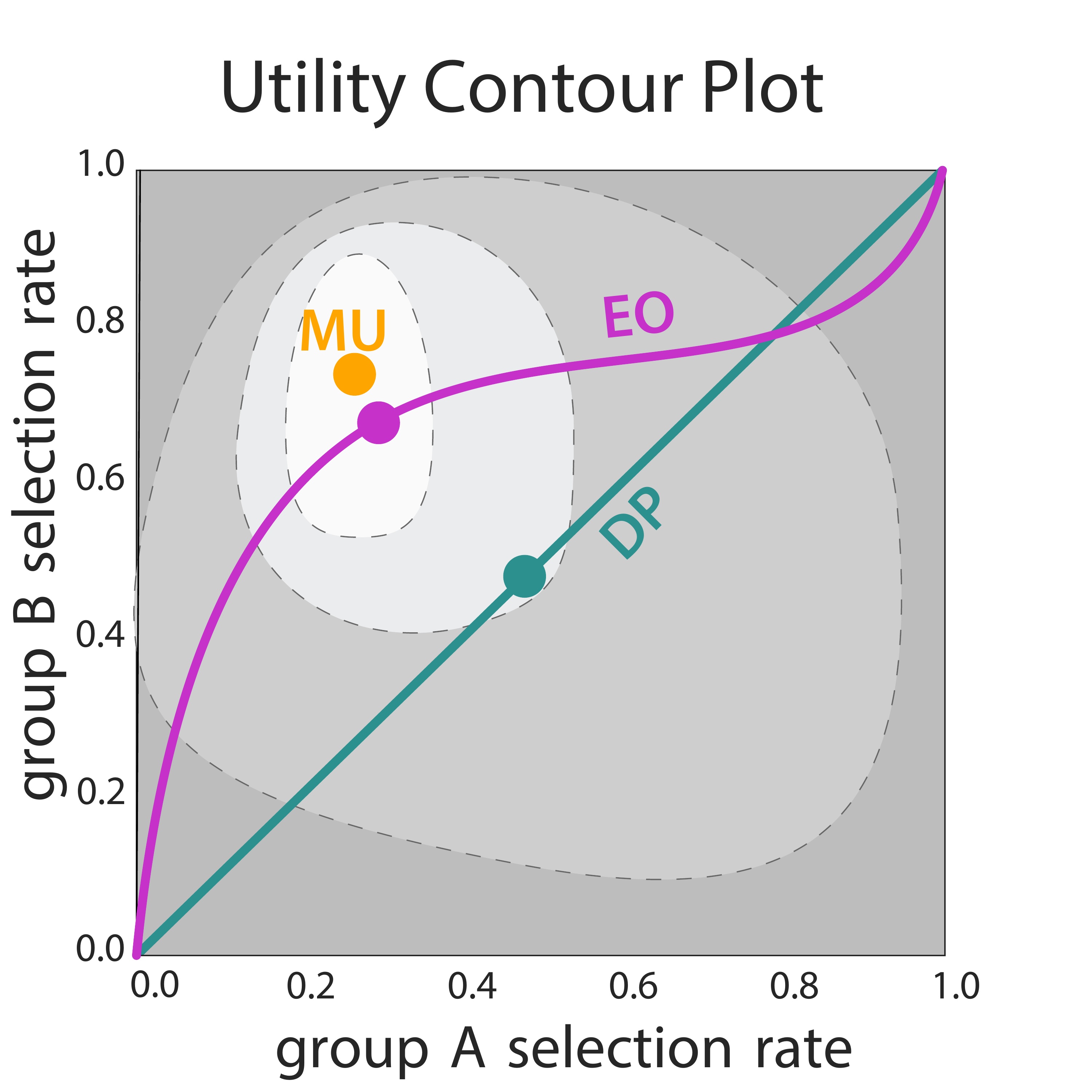}

\caption{ \label{fig:contour} Considering the utility as a function of selection rates, fairness constraints correspond to restricting the optimization to one-dimensional curves. The $\dempar$ (DP) constraint is a straight line with slope $1$, while the $\eqop$ (EO) constraint is a curve given by the graph of $\transfer$. 
The derivatives considered throughout Section~\ref{sec:main_thm_proofs} are taken with respect to the selection rate $\accratea$ (horizontal axis); projecting the EO and DP constraint curves to the horizontal axis recovers concave utility curves such as those shown in the lower panel of Figure~\ref{fig:outcome_util_curves} (where $\maxprof$ in is represented by a horizontal line through the MU optimal solution). }

\end{figure*}

	\subsection{A Characterization Theorem for $\dempar$}\label{sec:dem_par_proof}
		In this section, we provide a theorem that gives an explicit characterization for the range of selection rates $\accratea$ for $\popa$ when the bank loans according to $\dempar$. 
		Observe that the $\dempar$ objective corresponds to solving the following linear program:
		\begin{eqnarray*}\label{eq:hard_const_dempar}
		\max_{\pol = (\pola,\polb) \in [0,1]^{2\maxscore}}\Util(\pol) \quad \subto \quad \langle \dista, \pola \rangle = \langle \distb, \polb \rangle \:.
		\end{eqnarray*}
		Let us introduce the auxiliary variable $\accrate := \langle \dista, \pola \rangle = \langle \distb, \polb \rangle$ corresponding to the selection rate which is held constant across groups, so that all feasible solutions lie on the green DP line in Figure~\ref{fig:contour}. We can then express the following equivalent linear program:
		\begin{eqnarray*}\label{eq:hard_const_dempar_subst}
		\max_{\pol = (\pola,\polb) \in [0,1]^{2\maxscore},\accrate \in [0,1]}\Util(\pol) \quad \subto \quad \accrate = \langle \distj, \polj \rangle,~ \popj \in \{\popa,\popb\}\:.
		\end{eqnarray*}
		This is equivalent because, for a given $\accrate$, Proposition \ref{prop:opt_threshold_fair} says that the utility maximizing policies are of the form $\polj = \ratef_{\distj}^{-1}(\accrate)$. We now prove this:
		
		\begin{proof}[Proof of Proposition~\ref{prop:opt_threshold_fair} for $\dempar$] 
				Noting that $\ratef_{\distj}(\polj) = \langle \distj , \polj \rangle$, we see that, by Lemma~\ref{lem:optimal_best}, under the special case where $\boldv(x) = \util(x)$ and $\boldw(x) = 1$, the optimal solution $(\pola^*(\accrate),\polb^*(\accrate))$ for fixed $\ratef_{\dista}(\pola) = \ratef_{\distb}(\polb) = \accrate$ can be chosen to coincide with the threshold policies. Optimizing over $\accrate$, the global optimal must coincide with thresholds. 
				
		\end{proof}

		Hence, any optimal policy is equivalent to the threshold policy $\pol = (\ratef_{\dista}^{-1}(\accrate),\ratef_{\distb}^{-1}(\accrate))$, where $\accrate$ solves the following optimization:
		\begin{eqnarray}\label{eq:hard_const_dempar_accrate}
		\max_{\accrate \in [0,1]}\Util\left(\left(\ratef_{\dista}^{-1}(\accrate),\ratef_{\distb}^{-1}(\accrate)\right)\right) \:.
		\end{eqnarray}
		We shall show that the above expression is in fact a \emph{concave} function in $\accrate$, and hence the set of optimal selection rates can be characterized by first order conditions. This is presented formally in the following theorem:

	\begin{thm}[Selection rates for $\dempar$]\label{thm:dem_par_selection}
	
	The set of optimal selection rates $\accrate^*$ satisfying~\eqref{eq:hard_const_dempar_accrate} forms a continuous interval $[\accrate_{\dempar}^-,\accrate_{\dempar}^+]$, such that for any $\accrate \in [0,1]$, we have
	\begin{eqnarray*}
	\accrate < \accrate_{\dempar}^- &\text{if}& \fraca\util\left(\RCDF\supa(\accrate)\right) + \fracb\util\left(\RCDF\supb(\accrate)\right) > 0\:, \\
	\accrate > \accrate_{\dempar}^+ &\text{if}& \fraca\util\left(\RCDFpl\supa(\accrate) \right)+ \fracb\util\left(\RCDFpl\supb(\accrate)\right) < 0\:. \\
	\end{eqnarray*}
		\end{thm}
		\begin{proof}  Note that we can write 
		\begin{eqnarray*}
		\Util\left(\left(\ratef_{\dista}^{-1}(\accrate),\ratef_{\distb}^{-1}(\accrate)\right)\right)  &= \fraca\langle \util, \dista \circ  \ratef_{\dista}^{-1}(\accrate) \rangle + \fracb\langle \util, \distb \circ  \ratef_{\distb}^{-1}(\accrate) \rangle \:.
		\end{eqnarray*}

		Since $\util(x)$ is non-decreasing in $x$, Proposition~\ref{prop:beta_concavity} implies that $\beta \mapsto \Util\left(\left(\ratef_{\dista}^{-1}(\accrate),\ratef_{\distb}^{-1}(\accrate)\right)\right)$ is concave in $\accrate$. Hence, all optimal selection rates $\accrate^*$ lie in an interval $[\beta^-,\beta^+]$. To further characterize this interval, let us us compute left- and right-derivatives.
		\begin{eqnarray*}
		\partial_+ \Util\left(\left(\ratef_{\dista}^{-1}(\accrate),\ratef_{\distb}^{-1}(\accrate)\right)\right) &=& 
		\partial_+ \fraca\langle \util, \dista \circ \ratef_{\dista}^{-1}(\accrate) \rangle + \partial_+ \fracb  \langle \util , \distb \circ \ratef_{\distb}^{-1}(\accrate) \rangle\\
		&=& \fraca\langle \util, \partial_+ \left( \dista \circ \ratef_{\dista}^{-1}(\accrate) \right)\rangle + \fracb \langle \util , \partial_+  \left( \distb \circ\ratef_{\distb}^{-1}(\accrate) \right) \rangle\\
		&\overset{\text{Lemma}~\ref{lem:der_comp}}{=}& \fraca\langle \util, \bolde_{\RCDF\supa(\beta)}  \rangle + \fracb \langle \util ,  \bolde_{\RCDF\supb(\beta)}  \rangle\\
		&=& \fraca \util(\RCDF\supa(\beta))  + \fracb \util(\RCDF\supb(\beta))\:.
		\end{eqnarray*}
		The same argument shows that 
		\begin{eqnarray*}
		\partial_- \Util((\ratef_{\dista}^{-1}(\accrate),\ratef_{\distb}^{-1}(\accrate))) = \fraca \util(\RCDFpl\supa(\beta))  + \fracb \util(\RCDFpl\supb(\beta)).	\end{eqnarray*} 
		By concavity of $\Util\left(\left(\ratef_{\dista}^{-1}(\accrate),\ratef_{\distb}^{-1}(\accrate)\right)\right)$, a positive right derivative at $\accrate$ implies that $\accrate <  \accrate^*$ for all $\beta^*$ satisfying~\eqref{eq:hard_const_dempar_accrate}, and similarly, a negative left derivative at $\accrate$ implies that $\accrate >  \accrate^*$ for all $\beta^*$ satisfying ~\eqref{eq:hard_const_dempar_accrate}.
	
		\end{proof}
		
		With a result of the above form, we can now easily prove statements such as that in Corollary \ref{cor:dp_overeager} (see appendix \ref{app:main_results_proofs} for proofs), by fixing a selection rate of interest (e.g. $\beta_0$) and inverting the inequalities in Theorem \ref{thm:dem_par_selection} to find the exact population proportions under which, for example, $\dempar$ results in a higher selection rate than $\beta_0$.

	\subsection{$\eqop$ and General Constraints} \label{sec:eqop_proof}
		Next, we will provide a theorem that gives an explicit characterization for the range of selection rates $\accratea$ for $\popa$ when the bank loans according to $\eqop$. 
		Observe that the $\eqop$ objective corresponds to solving the following linear program:
		\begin{eqnarray}\label{eq:hard_const_eqop}
		\max_{\pol = (\pola,\polb) \in [0,1]^{2\maxscore}}\Util(\pol) \quad \subto \quad \langle \boldwa \circ  \dista, \pola \rangle = \langle \boldwb \circ  \distb, \polb \rangle~,
		\end{eqnarray}
		where $\boldw\supj = \frac{\pb}{\langle \pb, \distj\rangle}$. This problem is similar to the demographic parity optimization in~\eqref{eq:hard_const_dempar_accrate}, except for the fact that the constraint includes the weights. 
		Whereas we parameterized demographic parity solutions in terms of the acceptance rate $\accrate$ in equation~\eqref{eq:hard_const_dempar_accrate}, we will parameterize equation~\eqref{eq:hard_const_eqop} in terms of the true positive rate (TPR), $t:= \langle \boldwa \circ  \dista, \pola \rangle$. Thus,~\eqref{eq:hard_const_eqop} becomes 
		\begin{eqnarray}\label{eq:hard_const_eqop_2}
		\max_{t \in [0,t_{\max}]} \max_{(\pola,\polb) \in [0,1]^{2\maxscore}}\sum_{j \in \pops}\fracj \Util\supj(\polj) \quad \subto \quad \langle \boldwj \circ  \distj, \polj \rangle = t,~\popj\in\pops~,
		\end{eqnarray}
		where $t_{\max} = \min_{\popj \in \pops}\{\langle \distj, \boldwj \rangle\}$ is the largest possible TPR. The magenta EO curve in Figure~\ref{fig:contour} illustrates that feasible solutions to this optimization problem lie on a curve parametrized by $t$.
		Note that the objective function decouples for $\popj \in \pops$ for the inner optimization problem,
		\begin{eqnarray}\label{eq:hard_const_eqop_3}
		\max_{\polj \in [0,1]^{\maxscore}}\sum_{j \in \pops}\fracj \Util\supj(\polj) \quad \subto \quad \langle \boldwj \circ  \distj, \polj \rangle = t\:.
		\end{eqnarray}
		We will now show that all optimal solutions for this inner optimization problem are $\distj$-a.e. equal to a policy in $\Monotone(\distj)$, and thus can be written as $\ratef_{\distj}^{-1}(\accrate\supj)$, depending only on the resulting selection rate.
		
		\begin{proof}[Proof of Proposition~\ref{prop:opt_threshold_fair} for $\eqop$] 
			We apply Lemma~\ref{lem:optimal_best} to the inner optimization in~\eqref{eq:hard_const_eqop_3} with $\boldv(x) = \util(x)$ and $\boldw(x) = \frac{\pb(x)}{\langle \pb, \distj\rangle}$. The claim follows from the assumption that $\util(x)/\pb(x)$ is increasing by optimizing over $t$.
		\end{proof}

		This selection rate $\accrate\supj$ is uniquely determined by the TPR $t$ (proof appears in Appendix~\ref{sec:char_fairness_solns_eqopp}):
		
		\begin{restatable}{lem}{bijectiontpr}
		\label{lem:bijection_tpr_accrate}
Suppose that $\boldw(x) > 0$ for all $x$. Then the function
		\begin{eqnarray*}
		\constfj(\accrate):=	\langle \ratef_{\distj}^{-1}(\accrate),\distj \circ \boldwj \rangle
		\end{eqnarray*}
		is a bijection from $[0,1]$ to $[0,\langle \distj, \boldw \rangle]$. 
		\end{restatable}
		Hence, for any $t \in [0,t_{\max}]$, the mapping from TPR to acceptance rate, $\constfj^{-1}(t)$, is well defined and any solution to~\eqref{eq:hard_const_eqop_3} is $\distj$-a.e. equal to the policy $\ratef_{\distj}^{-1}(\constfj^{-1}(t))$. Thus~\eqref{eq:hard_const_eqop_2} reduces to 
		\begin{eqnarray}\label{eq:hard_const_eqop_accrate}
		\max_{t \in [0,t_{\max}]} \sum_{j \in \pops}\fracj \Util\supj\left(\ratef_{\distj}^{-1}\left(\constfj^{-1}(t)\right)\right)\:.
		\end{eqnarray}

	The above expression parametrizes the optimization problem in terms of a single variable. We shall show that the above expression is in fact a \emph{concave} function in $t$, and hence the set of optimal selection rates can be characterized by first order conditions. This is presented formally in the following theorem:
\begin{thm}[Selection rates for $\eqop$]\label{thm:eqop_selection}
	The set of optimal selection rates $\accrate^*$ for group~$\popa$ satsifying~\eqref{eq:hard_const_eqop_2} forms a continuous interval $[\accrate_{\eqop}^-,\accrate_{\eqop}^+]$, such that for any $\accrate \in [0,1]$, we have
	\begin{eqnarray*}
		\accrate < \accrate_{\eqop}^- &\text{if}& \fraca\frac{\util(\RCDF\supa(\accrate))}{\boldwa(\RCDF\supa(\accrate))} + \fracb\frac{\util(\RCDF\supb(\transfer_{\boldw}(\accrate)))}{\boldwb(\RCDF\supb(\transfer_{\boldw}(\accrate)))} > 0 \:,\\
		\accrate > \accrate_{\eqop}^+ &\text{if}& \fraca\frac{\util(\RCDFpl\supa(\accrate))}{\boldwa(\RCDFpl\supa(\accrate))} + \fracb\frac{\util(\RCDFpl\supb(\transfer_{\boldw}(\accrate)))}{\boldwb(\RCDFpl\supb(\transfer_{\boldw}(\accrate)))} < 0 ~.
		\end{eqnarray*}
		Here, $\transfer_{\boldw}(\accrate) := \constf_{\popb,\boldwb}^{-1}(\constf_{\popa,\boldwa}^{-1}(\accrate))$ denotes the (well-defined) map from selection rates $\accratea$ for $\popa$ to the selection rate $\accrateb$ for $\popb$ such that the policies $\pola^* := \ratef_{\dista}^{-1}(\accratea)$ and $\polb^* := \ratef_{\distb}^{-1}(\accrateb)$ satisfy the constraint in~\eqref{eq:hard_const_eqop}. 
		\end{thm}
		\begin{proof}  Starting with the equivalent problem in~\eqref{eq:hard_const_eqop_accrate}, we use the concavity result of Lemma~\ref{lem:const_to_util_derv}. Because the objective function is the positive weighted sum of two concave functions, it is also concave. Hence, all optimal true positive rates $t^*$ lie in an interval $[t^-,t^+]$. To further characterize $[t^-,t^+]$, we can compute left- and right-derivatives, again using the result of Lemma~\ref{lem:const_to_util_derv}.
		\begin{eqnarray*}
		\partial_+ \sum_{j \in \pops}\fracj \Util\supj\left(\ratef_{\distj}^{-1}(\constfj^{-1}(t))\right) &=& 
		\fraca \partial_+\Util\supa\left(\ratef_{\dista}^{-1}(\constfa^{-1}(t))\right) + \fraca \partial_+\Util\supa\left(\ratef_{\dista}^{-1}(\constfa^{-1}(t))\right)\\
		&=& \fraca\frac{\util(\RCDF\supa(\constfa^{-1}(t)))}{\boldwa(\RCDF\supa(\constfa^{-1}(t)))} + \fracb\frac{\util(\RCDF\supb(\constfb^{-1}(t)))}{\boldwb(\RCDF\supb(\constfb^{-1}(t)))}
		\end{eqnarray*}
		The same argument shows that 
		\begin{eqnarray*}
		\partial_- \sum_{j \in \pops}\fracj \Util\supj\left(\ratef_{\distj}^{-1}(\constfj^{-1}(t))\right) = \fraca\frac{\util(\RCDFpl\supa(\constfa^{-1}(t))}{\boldwa(\RCDFpl\supa(\constfa^{-1}(t)))} + \fracb\frac{\util(\RCDFpl\supb(\constfb^{-1}(t)))}{\boldwb(\RCDFpl\supb(\constfb^{-1}(t)))}.	
		\end{eqnarray*} 
		By concavity, a positive right derivative at $t$ implies that $t <  t^*$ for all $t^*$ satisfying~\eqref{eq:hard_const_eqop_accrate}, and similarly, a negative left derivative at $t$ implies that $t >  t^*$ for all $t^*$ satisfying~\eqref{eq:hard_const_eqop_accrate}.

		Finally, by Lemma~\ref{lem:bijection_tpr_accrate}, this interval in $t$ uniquely characterizes an interval of acceptance rates. Thus we translate directly into a statement about the selection rates $\accrate$ for group~$\popa$ by seeing that $\constfa^{-1}(t)=\accrate$ and $\constfb^{-1}(t) = \transfer_{\boldw}(\accrate)$. 
		\end{proof}

		Lastly, we remark that the results derived in this section go through verbatim for any linear constraint of the form $\langle \boldw, \dista \circ \pola \rangle = \langle \boldw, \distb \circ \polb \rangle$, as long as $\util(x)/\boldw(x)$ is increasing in $x$, and $\boldw(x) > 0$.

%% file: 05_experiments.tex
\section{Simulations}

We examine the outcomes induced by fairness constraints in the context of FICO scores for two race groups. FICO scores are a proprietary classifier widely used in the United States to predict credit worthiness. Our FICO data is based on a sample of 301,536 TransUnion TransRisk scores from 2003 \citep{fed07}, preprocessed by \citet{hardt16equality}. These scores, corresponding to $x$ in our model, range from 300 to 850 and are meant to predict credit risk. Empirical data labeled by race allows us to estimate the distributions $\pi\supj$, where $\popj$ represents race, which is restricted to two values: white non-Hispanic (labeled ``white" in figures), and black. Using national demographic data, we set the population proportions to be $18\%$ and $82\%$.

Individuals were labeled as defaulted if they failed to pay a debt for at least 90 days on at least one account in the ensuing 18-24 month period; we use this data to estimate the success probability given score, $\pb\supj(x)$, which we allow to vary by group to match the empirical data (see Figure~\ref{fig:empirical_payback_probs}). Our outcome curve framework allows for this relaxation; however, this discrepancy can also be attributed to group-dependent mismeasurement of score, and adjusting the scores accordingly would allow for a single $\pb(x)$. We use the success probabilities to define the affine utility and score change functions defined in Example~\ref{ex:credit_lending}. We model individual penalties as a score drop of $c_-=-150$ in the case of a default, and in increase of $c_+=75$ in the case of successful repayment.

\begin{figure*}[t] 
   \centering
   \includegraphics[width=0.7\textwidth]{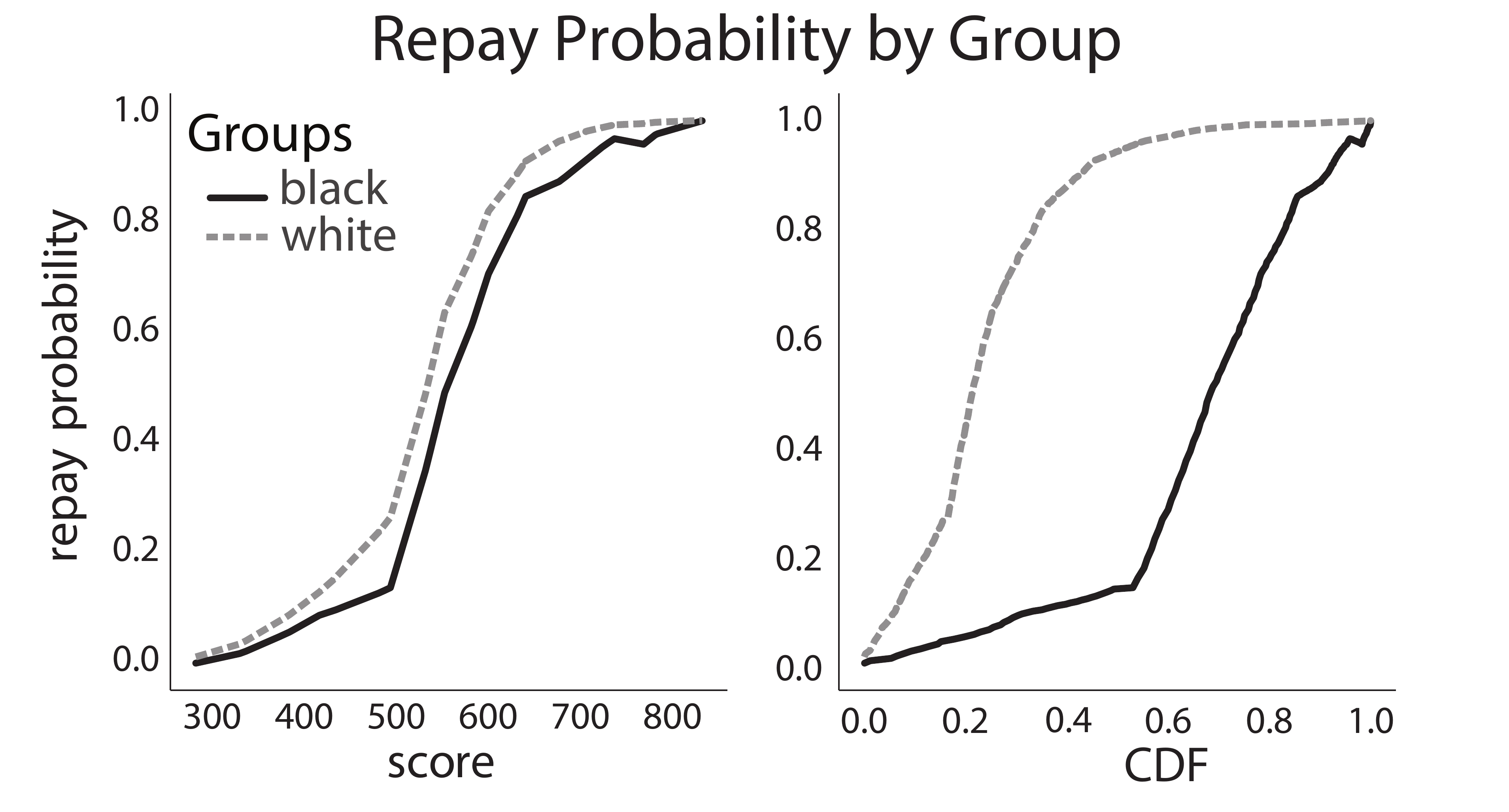}
   \caption{The empirical payback rates as a function of credit score and CDF for both groups from the TransUnion TransRisk dataset.}
   \label{fig:empirical_payback_probs}
\end{figure*}

\begin{figure*}[t] 
   \centering
   \includegraphics[width=0.7\textwidth]{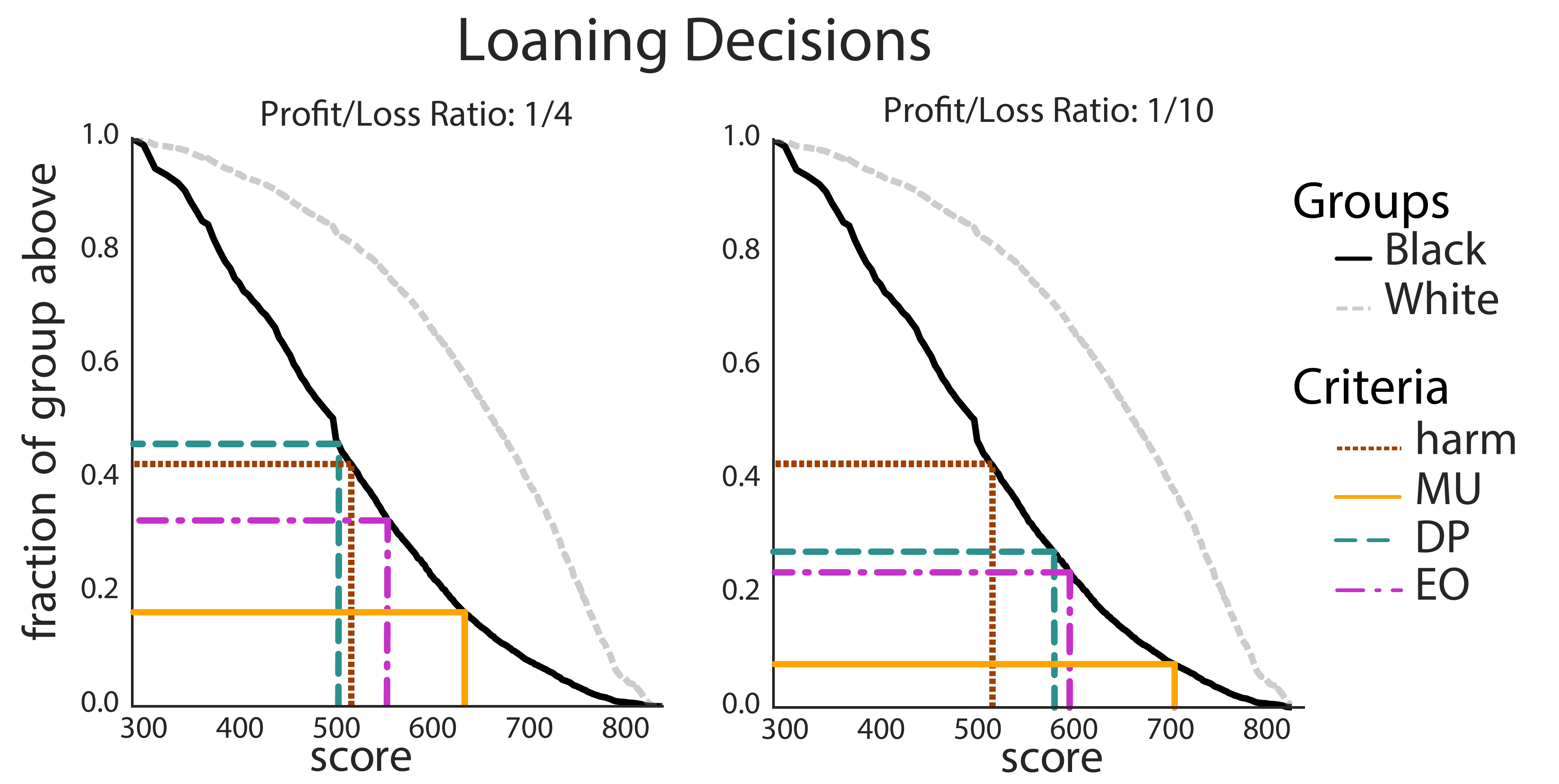}
   \caption{The empirical CDFs of both groups are plotted along with the decision thresholds resulting from $\maxprof$, $\dempar$, and $\eqop$ for a model with bank utilities set to (a) $\frac{u_-}{u_+} = -4$ and (b) $\frac{u_-}{u_+} = -10$. The threshold for active harm is displayed; in (a) $\dempar$ causes active harm while in (b) it does not. $\eqop$ and $\maxprof$ never cause active harm.}
   \label{fig:empirical_outcome1}
\end{figure*}

In Figure~\ref{fig:empirical_outcome1}, we display the empirical CDFs along with selection rates resulting from different loaning strategies for two different settings of bank utilities. In the case that the bank experiences a loss/profit ratio of $\frac{u_-}{u_+} = -10$, no fairness criteria surpass the active harm rate $\accrate_0$; however, in the case of $\frac{u_-}{u_+} = -4$, $\dempar$ overloans, in line with the statement in Corollary~\ref{cor:dp_overeager}.

These results are further examined in Figure~\ref{fig:empirical_outcome}, which displays the normalized outcome curves and the utility curves for both the white and the black group. To plot the $\maxprof$ utility curves, the group that is not on display has selection rate fixed at $\beta^{\maxprof}$. In this figure, the top panel corresponds to the average change in credit scores for each group under different loaning rates $\accrate$; the bottom panels shows the corresponding \emph{total} utility $\Util$ (summed over both groups and weighted by group population sizes) for the bank. 

Figure~\ref{fig:empirical_outcome} highlights that the position of the utility optima in the lower panel determines the loan (selection) rates. In this specific instance, the utility and change ratios are fairly close, $\frac{u_-}{u_+} = -4$, and $\frac{c_-}{c_+} = -2$, meaning that the bank's profit motivations align with individual outcomes to some extent. Here, we can see that $\eqop$ loans much closer to optimal than $\dempar$, similar to the setting suggested by Corollary~\ref{cor:fairness_rel_imp}.

Although one might hope for decisions made under fairness constraints to positively affect the black group, we observe the opposite behavior. The $\maxprof$ policy (solid orange line) and the $\eqop$ policy result in similar expected credit score change for the black group. However, $\dempar$ (dashed green line) causes a negative expected credit score change in the black group, corresponding to active harm. For the white group, the bank utility curve has almost the same shape under the fairness criteria as it does under $\maxprof$, the main difference being that fairness criteria lowers the total expected profit from this group.

This behavior stems from a discrepancy in the outcome and profit curves for each population. While incentives for the bank and positive results for individuals are somewhat aligned for the majority group, under fairness constraints, they are more heavily misaligned in the minority group, as seen in graphs (left) in Figure~\ref{fig:empirical_outcome}. We remark that in other settings where the \emph{unconstrained} profit maximization is misaligned with individual outcomes (e.g., when $\frac{u_-}{u_+} = -10$), fairness criteria may perform more favorably for the minority group by pulling the utility curve into a shape consistent with the outcome curve.

By analyzing the resulting affects of $\maxprof$, $\dempar$, and $\eqop$ on actual credit score lending data, we show the applicability of our model to real-world applications. In particular, some results shown in Section~\ref{sec:results} hold empirically for the FICO TransUnion TransRisk scores.

\begin{figure}[p] 
	\centering
	\includegraphics[width=.9\columnwidth]{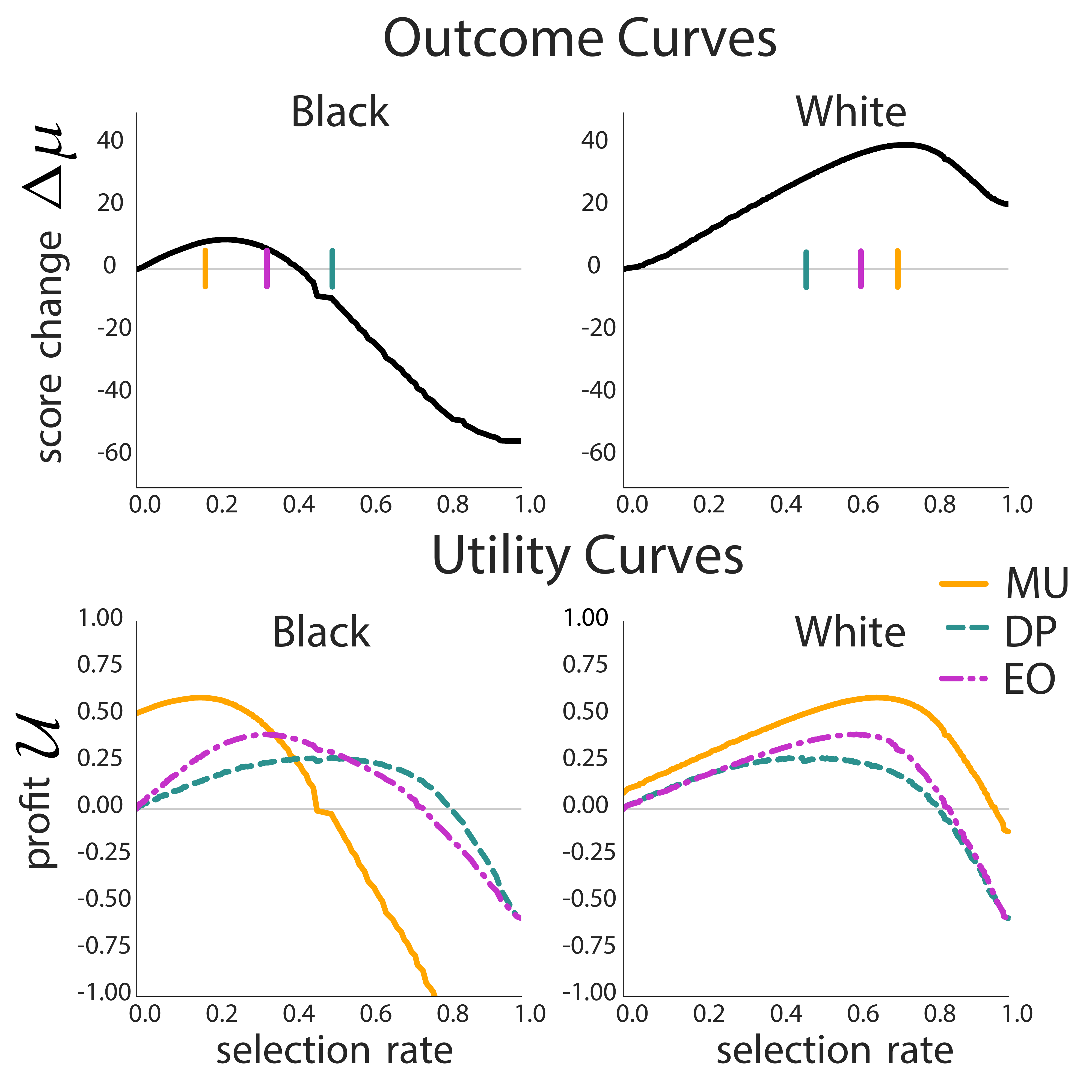}
	\caption{The outcome and utility curves are plotted for both groups against the group selection rates. The relative positions of the utility maxima determine the position of the decision rule thresholds. We hold $\frac{u_-}{u_+} = -4$ as fixed. 
	}
	\label{fig:empirical_outcome}
\end{figure}

%% file: 06_conclusion.tex

\section{Conclusion and Future Work}

We argue that without a careful model of delayed outcomes, we
cannot foresee the impact a fairness criterion would have if enforced as a
constraint on a classification system. 
However, if such an accurate outcome model is available, we show that there are more direct 
ways to optimize for positive outcomes than via existing fairness criteria.

Our formal framework exposes a concise, yet expressive way to model outcomes via the expected
change in a variable of interest caused by an institutional decision. This leads to the natural
concept of an outcome curve that allows us to interpret and compare solutions
effectively.
In essence, the formalism we propose requires us to understand the two-variable
causal mechanism that translates decisions to outcomes. Depending on the
application, such an understanding might necessitate greater domain knowledge
and additional research into the specifics of the application. This is
consistent with much scholarship that points to the context-sensitive nature of
fairness in machine learning.

An interesting direction for future work is to consider other characteristics of
impact beyond the change in population \emph{mean}. Variance and
individual-level outcomes are natural and important considerations. Moreover, it
would be interesting to understand the robustness of outcome optimization to
modeling and measurement errors.

%% file: acknow.tex
\section*{Acknowledgements}
We thank Lily Hu, Aaron Roth, and Cathy O'Neil for discussions and feedback on an earlier version of the manuscript. We thank the students of CS294: Fairness in Machine Learning (Fall 2017, University of California, Berkeley) for inspiring class discussions and comments on a presentation that was a precursor of this work. This material is based upon work supported by the National Science Foundation Graduate Research Fellowship under Grant No. DGE 1752814.

%% file: 10_threshold_policies_appendix.tex
\section{Optimality of Threshold Policies}

\subsection{Proof of Lemma~\ref{lem:pol_sim_lemma}\label{sec:lem_pol_sim}} 
	
	We begin with the first statement of the lemma. Suppose $\polj \simj \polj'$. Then there exists a set $\calS \subset \X$ such that $\distj(x) = 0$ for all $x \in \calS$, and for all $x \notin \calS$, $\polj(x) = \polj'(x)$. Thus, 
	\begin{eqnarray*}
	\ratef_{\dist}(\polj) - \ratef_{\distj}(\polj') &=& \sum_{x \in \X}\distj(x)(\polj(x) - \polj'(x)) \\
	&=& \sum_{x \in \calS}\distj(x)(\polj(x) - \polj'(x)) = 0\:. \\
	\end{eqnarray*}
	Conversely, suppose that $\ratefj(\polj) = \ratefj(\polj')$. 
	Let $\polj = \pol_{c,\gamma}$ and $\polj' = \pol_{c',\gamma'}$ as in Definition~\ref{def:monotone_pols}. We now have the following cases:
	\begin{enumerate}
		\item Case 1: $c = c'$. Then $\polj(x) = \polj'(x)$ for all $x \in \X -\{c\}$. Hence, 
		\begin{eqnarray*}
		0 = \ratef_{\dist}(\polj) - \ratef_{\distj}(\polj') = \dist(x)(\polj(c) - \polj'(c)) \:.
		\end{eqnarray*}
		This implies that either $\polj(c) = \polj'(c)$, and thus $\polj(x) = \polj'(x)$ for all $x \in \X$, or otherwise $\dist(c) = 0$, in which case we still have $\polj \simj \polj'$ (since the two policies agree every outside the set $\{c\}$). 
		\item Case 2: $c \ne c'$. We assume assume without loss of generality that $c' < c \le \maxscore$. Since the policies $\pol_{c',1}$ and $\pol_{c'+1,0}$ are identity for $c' < \maxscore$, we may also assume without loss of generality that $\gamma' \in [0,1)$. Thus for all $x \in \calS:= \{c',c'+1,\dots,C\}$, we have $\polj'(x) < \polj(x)$. This implies that 
		\begin{eqnarray*}
		0 &=& \ratef_{\dist}(\polj) - \ratef_{\distj}(\polj') \\
		&=& \sum_{x \in \calS} \distj(x)(\polj(x) - \polj'(x)) \\
		&\ge& \min_{x \in \calS}(\polj(c) - \polj'(x)) \cdot  \sum_{x \in \calS} \dist(x)\:.
		\end{eqnarray*}
		Since $\min_{x \in \calS}(\polj(c) - \polj'(x))  > 0$, it follows that $\sum_{x \in \calS} \distj(x) = 0$, whence $\polj \simj \polj'$.  
		\end{enumerate}

		Next, we show that $\ratef_{\dist}$ is a bijection from $\Monotone(\dist) \to [0,1]$. That $\ratef_{\dist}$ is injective follows immediately from the fact if $\ratefj(\pol) = \ratefj(\polj')$, then $\polj \simj \polj'$. To show it is surjective, we exhibit for every $\accrate \in [0,1]$ a threshold policy $\pol_{c,\gamma}$ for which $\ratefj(\pol_{c,\gamma}) = \accrate$. We may assume $\accrate < 1$, since the all-ones policy has a selection rate of $1$.

		Recall the definition of the inverse CDF 
		\begin{eqnarray*}
		\RCDFj(\accrate) := \argmax\{c:\sum_{x=c}^C\dist(x) > \accrate\}\:.
		\end{eqnarray*}
		Since $\accrate < 1$, $\RCDFj(\accrate) \le \maxscore$. Let $\accrate_+ = \sum_{x=\RCDFj(\accrate)}^C\dist(x)$, and let $\accrate_- = \sum_{x=\RCDFj(\accrate) + 1}^C\dist(x)$. Note that by definition, we have $\accrate_- \le \accrate < \accrate_+$, and $\accrate_+ - \accrate_- = \dist(\RCDFj(\accrate))$. Hence, if we define $\gamma = \frac{\accrate - \accrate_-}{\accrate_+ - \accrate_-}$, we have 
		\begin{eqnarray*}
		\ratef_{\distj}(\pol_{\RCDFj(\accrate), \gamma}) =  \dist(\RCDFj(\accrate)) \gamma + \sum_{x=\RCDFj(\accrate) + 1}^C\dist(x) = \accrate_- + (\accrate_+ - \accrate_-)\gamma = \accrate_- + \accrate - \accrate_- = \accrate\:.
		\end{eqnarray*}

\subsection{Proof of Lemma~\ref{lem:optimal_best}~\label{sec:optimal_best_proof}}
	Given $\pol \in [0,1]^{\maxscore}$, 
	we define the \emph{normal cone} at $\pol$ as $\NC(\pol):= \mathrm{ConicalHull}\{\boldz: \pol + \boldz \in [0,1]^{\maxscore}\}$. 
	We can describe $\NC(\pol)$ explicitly as:
	\begin{eqnarray*}\label{eq:normal_cone_def}
	 \NC(\pol) := \{\boldz \in \R^{\maxscore} : \boldz_i \le 0 \text{ if } \pol_i = 0, \boldz_i \ge 0 \text{ if } \pol_i = 1\}\:.
	\end{eqnarray*}
	Immediately from the above definition, we have the following useful identity, which is that for any vector $\boldg \in \R^\maxscore$,
	\begin{eqnarray}\label{eq:normal_cone_first_order}
	\langle \boldg, \boldz \rangle \le 0 ~\forall \boldz \in \NC(\pol),  &\text{if and only if}& \forall x\in \X, \begin{cases} \pol(x) = 0 & \boldg(x) < 0\\
	\pol(x) = 1 & \boldg(x) > 0\\
	\pol(x) \in [0,1] & \boldg(x) = 0
	\end{cases}\:.
	\end{eqnarray}

	Now consider the optimization problem~\eqref{eq:optim}. By the first order KKT conditions, we know that for any optimizer $\pol_*$ of the above objective, there exists some $\lamhat \in \R$ such that, for all $\boldz \in \NC(\pol_*)$
	\begin{eqnarray*}
	\langle \boldz, \boldv \circ \dist + \lamhat \dist \circ \boldw \rangle \le 0\:.
	\end{eqnarray*}
	By~\eqref{eq:normal_cone_first_order}, we must have that
	\begin{eqnarray*}
	\pol_*(x) = \begin{cases} 0 & \dist(x)(\boldv(x) + \lamhat\boldw(x))< 0\\
	1 & \dist(x)(\boldv(x) + \lamhat\boldw(x)) > 0\\
	 \in [0,1] & \dist(x)(\boldv(x) + \lamhat\boldw(x)) = 0
	\end{cases}\:.
	\end{eqnarray*}
	Now $\pol_*(x)$ is not necessarily a threshold policy. To conclude the theorem, it suffices to exhibit a threshold policy $\widetilde{\pol}_*$ such that $\pol_*(x) \simpi \widetilde{\pol}_*$. (Note that $\widetilde{\pol}_*(x)$ will also be feasible for the constraint, and have the same objective value; hence $\widetilde{\pol}_*$ will be optimal as well.)

	Given $\pol_*$ and $\lamhat$, let $c_* = \min\{ c \in \X: \boldv(x) + \lamhat\boldw(x) \ge 0\}$. If either (a) $\boldw(x) = 0$ for all $x \in \X$ and $\boldv(x)$ is strictly increasing or (b) $\boldv(x)/\boldw(x)$ is strictly increasing, then the modified policy 
	\begin{eqnarray*}
	\widetilde{\pol}_*(x) = \begin{cases} 0 & x < c_*\\
	\pol_*(x) & x = c_* \\
	1 & x > c_*
	\end{cases}~,
	\end{eqnarray*}
	is a threshold policy, and $\pol_*(x) \simpi \widetilde{\pol}_*$. Moreover, $\langle \boldw,\widetilde{\pol}_* \rangle = \langle \boldw,\widetilde{\pol}_* \rangle$ and $\langle \dist,\widetilde{\pol}_* \rangle = \langle \dist,\widetilde{\pol}_* \rangle$, which implies that $\widetilde{\pol}_*$ is an optimal policy for the objective in Lemma~\ref{lem:optimal_best}. 

\subsection{Proof of Lemma~\ref{lem:der_comp}~\label{sec:der_comp_proof}}
	We shall prove 
	\begin{eqnarray}
	\partial_+ \left(\distj \circ \ratef_{\distj}^{-1}(\accrate)\right) = \bolde_{\RCDFj(\beta)}~,
	\end{eqnarray}
	where the derivative is with respect to $\beta$.
	The computation of the left-derivative is analogous. Since we are concerned with right-derivatives, we shall take $\accrate \in [0,1)$.
	Since $\distj \circ \ratef_{\distj}^{-1}(\accrate)$ does not depend on the choice of representative for $\ratef_{\distj}^{-1}$, we can choose a cannonical representation for $\ratef_{\distj}^{-1}$. 
	In Section~\ref{sec:lem_pol_sim}, we saw that the threshold policy $\pol_{\RCDFj(\accrate),\gamma(\accrate)}$ had acceptance rate $\accrate$, where we had defined 
	\begin{eqnarray}
	&&\accrate_+ = \sum_{x=\RCDFj(\accrate)}^C\dist(x)~\text{ and }~\accrate_- = \sum_{x=\RCDFj(\accrate) + 1}^C\dist(x)\:,\\
	&&\gamma(\accrate) = \frac{\accrate - \accrate_-}{\accrate_+ - \accrate_-}~.
	\end{eqnarray}
	Note then that for each $x$, $\pol_{\RCDFj(\accrate),\gamma(\accrate)}(x)$ is piece-wise linear, and thus admits left and right derivatives. We first claim that
	\begin{eqnarray}\label{eq:right_der_not_ex}
	\forall x \in \X \setminus \{\RCDFj(\accrate)\},~~\partial_+ \pol_{\RCDFj(\accrate),\gamma(\accrate)}(x) = 0 \:.
	\end{eqnarray}
	To see this, note that $\RCDFj(\accrate)$ is right continuous, so for all $\epsilon$ sufficiently small, $\RCDFj(\accrate + \epsilon) = \RCDFj(\accrate)$. Hence, for all $\epsilon$ sufficiently small and all $x \ne \RCDF(\accrate)$, we have $\pol_{\RCDFj(\accrate+\epsilon),\gamma(\accrate + \epsilon)}(x) = \pol_{\RCDFj(\accrate+\epsilon),\gamma(\accrate + \epsilon)}(x)$, as needed. 
	Thus, Equation~\eqref{eq:right_der_not_ex} implies that $\partial_+ \distj \circ \ratefj^{-1}(\accrate)$ is supported on $x = \RCDFj(\accrate)$, and hence
	\begin{eqnarray*}
	\partial_+ \distj \circ \ratefj^{-1}(\accrate) = \partial_+ \distj(x)\pol_{\RCDFj(\accrate),\gamma(\accrate)}(x) \big{|}_{x = \RCDFj(\accrate)} \cdot \bolde_{\RCDFj(\accrate)}\:.
	\end{eqnarray*}

	To conclude, we must show that $\partial_+ \distj(x)\pol_{\RCDFj(\accrate),\gamma(\accrate)}(x) \big{|}_{x = \RCDFj(\accrate)} = 1$. To show this, we have
	\begin{eqnarray*}
	1 &=& \partial_+(\accrate)\\
	&=& \partial_+(\ratefj(\pol_{\RCDFj(\accrate),\gamma(\accrate)}))\quad\text{since}\quad\quad\ratefj(\pol_{\RCDFj(\accrate),\gamma(\accrate)}) = \accrate ~\forall \accrate \in [0,1)\\
	&=& \partial_+\left(\sum_{x \in \X} \dist(x) \cdot \pol_{\RCDFj(\accrate),\gamma(\accrate)}(x)\right)\\
	&=& \partial_+  \dist(x) \cdot \pol_{\RCDFj(\accrate),\gamma(\accrate)}(x)\big{|}_{x = \RCDFj(\accrate)}~, \text{ as needed. }
	\end{eqnarray*}

%% file: 10_characterizations.tex
	\section{Characterization of Fairness Solutions\label{sec:char_fairness_solns}}

\subsection{Derivative Computation for $\eqop$\label{sec:char_fairness_solns_eqopp}}
In this section, we prove Lemma~\ref{lem:bijection_tpr_accrate}, which we recall below.
\bijectiontpr*
We will prove Lemma~\ref{lem:bijection_tpr_accrate} in tandem with the following derivative computation which we applied in the proof of Theorem~\ref{thm:eqop_selection}.

\begin{lem}\label{lem:const_to_util_derv}
	The function
	\begin{eqnarray*}
	\Util\supj(t;\boldwj) := \Util\supj\left(\ratef_{\distj}^{-1}\left(\constfj^{-1}(t )\right)\right)
	\end{eqnarray*}
	is concave in $t$ and has left and right derivatives
	\begin{eqnarray*}
	\partial_+ \Util\supj(t;\boldwj) =  \frac{\util(\RCDFj(\constfj^{-1}(t)))}{\boldwj(\RCDFj(\constfj^{-1}(t)))} 
	&\text{and}& \partial_- \Util\supj(t;\boldwj) = \frac{\util(\RCDFplj(\constfj^{-1}(t)))}{\boldwj(\RCDFplj(\constfj^{-1}(t)))}\:.
	\end{eqnarray*}
	\end{lem}

\begin{proof}[Proof of Lemmas~\ref{lem:bijection_tpr_accrate}~and~\ref{lem:const_to_util_derv}] Consider a $\accrate \in [0,1]$. Then, $\distj \circ \ratef_{\distj}^{-1}(\accrate)$ is continuous and left and right differentiable by Lemma~\ref{lem:der_comp}, and its left and right derivatives are indicator vectors $\bolde_{\RCDFj(\accrate)}$ and $\bolde_{\RCDFplj(\accrate)}$, respectively. 
	Consequently, $\accrate \mapsto \langle \boldwj, \distj \circ \ratef_{\distj}^{-1}(\accrate) \rangle$ has left and right derivatives $\boldwj(\RCDF(\accrate))$ and $\boldwj(\RCDFpl(\accrate))$, 
	respectively; both of which are both strictly positive by the assumption $\boldw(x) > 0$. 
	Hence, $\constfj(\accrate) = \langle \boldwj, \distj \circ \ratef_{\distj}^{-1}(\accrate) \rangle$ is strictly increasing  in $\accrate$, and so the map is injective. 
	It is also surjective because $\accrate = 0$ induces the policy $\polj = \mathbf{0}$ and $\accrate = 1$ induces the policy $\polj = \mathbf{1}$ (up to $\distj$-measure zero). Hence, $\constfj(\accrate)$ is an order preserving bijection with left- and right-derivatives, and we can compute the left and right derivatives of its inverse as follows:
	\begin{eqnarray*}
	\partial_+ \constfj^{-1}(t) = \frac{1}{\partial_+ \constfj(\accrate) \big{|}_{\accrate = \constfj^{-1}(t)}} = \frac{1}{\boldwj(\RCDFj(\constfj^{-1}(t)))}~,
	\end{eqnarray*}
	and similarly, $\partial_- \constfj^{-1}(t) = \frac{1}{\boldwj(\RCDFpl(\constfj^{-1}(t)))}$. Then we can compute that
	\begin{eqnarray*}
	\partial_+ \Util\supj(\ratef_{\distj}(\constfj^{-1}(t))) &=& \partial_+\Util(\ratef_{\distj}(\accrate)) \big{|}_{\accrate = \constfj^{-1}(t))} \cdot \partial_+ \constfj(\sup(t))\\
	&=& \frac{\util(\RCDFj(\constfj^{-1}(t)))}{\boldwj(\RCDFj(\constfj^{-1}(t)))}~.
	\end{eqnarray*}
	and similarly $\partial_- \Util\supj(\ratef_{\distj}(\constfj(t))) = \frac{\Util(\RCDFplj(\constfj^{-1}(t)))}{\boldwj(\RCDFplj(\constfj^{-1}(t)))}$. One can verify that for all $t_1 < t_2$, one has that $\partial_+ \Util\supj(\ratef_{\distj}(\constfj^{-1}(t_1))) \ge \partial_- \Util\supj(\ratef_{\distj}(\constfj^{-1}(t_2)))$, and that for all $t$, $\partial_+ \Util\supj(\ratef_{\distj}(\constfj^{-1}(t))) \le \partial_- \Util\supj(\ratef_{\distj}(\constfj^{-1}(t)))$. These facts establish that the mapping $t\mapsto \Util\supj(\ratef_{\distj}(\constfj^{-1}(t)))$ is concave. 
	\end{proof}

%% file: 10_characterize_soft.tex
\subsection{Characterizations Under Soft Constraints}
		
		Given a convex penalty $\Phi:\R \to \R_{\ge 0}$, and $\lambda \in \R_{\ge 0}$, one can write down the general form for soft constrained utility optimization
		\begin{eqnarray}\label{eq:soft_const}
		\max_{\pol = (\pola,\polb)}\Util(\pol) - \lambda \Phi(\langle \boldwa \circ \dista, \pola \rangle - \langle \boldwb \circ \distb, \polb \rangle)~,
		\end{eqnarray} 
		where $\boldwa$ and $\boldwb$ represent generic constraints. Again, we shall assume that for $\popj \in \pops$, $\util(x)/\boldwj(x)$ is non-decreasing.  Recall that for $\boldwj = (1,1,\dots,1)$, one recovers the soft version of $\dempar$, whereas for $\boldwj = \frac{\pb}{\langle \pb, \distj\rangle}$, one recovers the soft constrained version of $\eqop$. 

		The same argument presented in Section~\ref{sec:eqop_proof} shows that the optimal policies are of the form
		\begin{eqnarray*}
		\polj = \ratef_{\distj}^{-1}(\constfj^{-1}(t\supj))~,
		\end{eqnarray*}
		where $(t\supa,t\supb)$ are solutions to the following optimization problem:  
		\begin{eqnarray*}
		\max_{t\supa \in [0,\langle \dista, \boldwa \rangle],t\supb \in [0,\langle \distb, \boldwb \rangle]}~\fraca \Utila(\ratef_{\dista}^{-1}(\constfa^{-1}(t\supa))) + \fracb\Utilb(\ratef_{\distb}^{-1}(\constfb^{-1}(t\supb))) - \lambda \Phi(t\supa - t\supb)\:.
		\end{eqnarray*}
		The following lemma gives us a first order characterization of these optimal TPRs, $(t\supa, t\supb)$.
		\begin{lem}\label{lem:first_order_op} All optimal policies are equivalent to threshold policies with selection rate $(\accrate\supa,\accrate\supb)$ which satisfy
		\begin{eqnarray}
		\begin{bmatrix} 0 \\
		0 
		\end{bmatrix} \in
		\begin{bmatrix}
	\left[ \frac{\util(\RCDF\supa(\accratea))}{\boldwa(\RCDF\supa(\accrateb))}  - \lambda  \partial_+  \Phi(\Delta) ,\frac{\util(\RCDFpl\supa(\accratea))}{\boldwa(\RCDFpl\supa(\accratea))} - \lambda  \partial_-  \Phi(\Delta) \right] \\
	\left[ \frac{\util(\RCDF\supb(\accrateb))}{\boldwb(\RCDF\supb(\accrateb))}  + \lambda \partial_- \Phi(\Delta), \frac{\util(\RCDFpl\supb(\accrate))}{\boldwb(\RCDFpl\supb(\accrateb))} + \lambda \partial_+ \Phi(\Delta)  \right]
	\end{bmatrix}~,
		\end{eqnarray}
		where $\Delta = t\supa - t\supb = \constfa(\accratea) - \constfb(\accrateb)$.

		\end{lem}

	\begin{proof}

	Let $\partial (\cdot)$ denote the super-gradient set of a concave function. Note that if $F$ is left-and-right differentiable and concave, then $\partial F(x) = [\partial_+ F(x), \partial_- F(x)]$. 
	By concavity of $\Utilj$ and convexity of $\Phi$, we must have that 
	\begin{eqnarray*}
	\begin{bmatrix}
	0 \\ 
	0 
	\end{bmatrix} &\in& \partial \sum_{\popj \in \pops}\Util\supj\left(\ratef_{\distj}^{-1}\left(\constfj^{-1}(t\supj)\right)\right) - \lambda \Phi(t\supa - t\supb) \\
	&=&
	\begin{bmatrix}
	\partial \Util\supa\left(\ratef_{\dista}^{-1}(\constfa^{-1}(t\supa))\right) +  \partial_{t\supa} \{- \lambda \Phi(t\supa - t\supb)\} \\ 
	\partial \Util\supa\left(\ratef_{\distb}^{-1}(\constfb^{-1}(t\supb))\right) + \partial_{t\supb} \{- \lambda \Phi(t\supa - t\supb)\} \\
	\end{bmatrix}\\
	&=& \begin{bmatrix}
	\partial \Util\supa\left(\ratef_{\dista}^{-1}(\constfa^{-1}(t\supa))\right) - \lambda  \partial  \Phi(t)\big{|}_{t = t\supa - t\supb} \\
	\partial \Util\supb\left(\ratef_{\distb}^{-1}(\constfb^{-1}(t\supb))\right) + \lambda \partial \Phi(t) \big{|}_{t = t\supa - t\supb} \\
	\end{bmatrix}\\
	&=& \begin{bmatrix}
	[\partial_+ \Util\supa\left(\ratef_{\dista}^{-1}(\constfa^{-1}(t\supa))\right) - \lambda  \partial_+  \Phi(t)\big{|}_{t = t\supa - t\supb} , \partial_- \Util\supa\left(\ratef_{\dista}^{-1}(\constfa^{-1}(t\supa))\right) - \lambda  \partial_-  \Phi(t)\big{|}_{t = t\supa - t\supb}] \\
	[\partial_+ \Util\supb\left(\ratef_{\distb}^{-1}(\constfb^{-1}(t\supb))\right) + \lambda \partial_- \Phi(t) \big{|}_{t = t\supa - t\supb}, \partial_- \Util\supb\left(\ratef_{\distb}^{-1}(\constfa^{-1}(t\supb))\right) + \lambda \partial_+ \Phi(t) \big{|}_{t = t\supa - t\supb}] \\
	\end{bmatrix}\\
	&=& \begin{bmatrix}
	\left[\frac{\util(\RCDF\supa(\constfa^{-1}(t\supa)))}{\boldwa(\RCDF\supa(\constfa^{-1}(t\supa)))} - \lambda  \partial_+  \Phi(t)\big{|}_{t = t\supa - t\supb} , \frac{\util(\RCDFpl\supa(\constfa^{-1}(t\supa)))}{\boldwa(\RCDFpl\supa(\constfa^{-1}(t\supa)))}  - \lambda  \partial_-  \Phi(t)\big{|}_{t = t\supa - t\supb}\right] \\
	\left[\frac{\util(\RCDF\supb(\constfb^{-1}(t\supb)))}{\boldwb(\RCDF\supb(\constfb^{-1}(t\supb)))} + \lambda \partial_- \Phi(t) \big{|}_{t = t\supa - t\supb}, \frac{\util(\RCDFpl\supb(\constfb^{-1}(t\supb)))}{\boldwb(\RCDFpl\supb(\constfb^{-1}(t\supb)))}  + \lambda \partial_+ \Phi(t) \big{|}_{t = t\supa - t\supb}\right] \\
	\end{bmatrix}\\
	&=& \begin{bmatrix}
	\left[\frac{\util(\RCDF\supa(\accratea))}{\boldwa(\RCDF\supa(\accratea))} - \lambda  \partial_+  \Phi(t)\big{|}_{t = t\supa - t\supb} , \frac{\util(\RCDFpl\supa(\accratea))}{\boldwa(\RCDFpl\supa(\accrateb))}  - \lambda  \partial_-  \Phi(t)\big{|}_{t = t\supa - t\supb}\right] \\
	\left[\frac{\util(\RCDF\supb(\accrate))}{\boldwb(\RCDF\supb(\accrateb))} + \lambda \partial_- \Phi(t) \big{|}_{t = t\supa - t\supb}, \frac{\util(\RCDFpl\supb(\accrateb))}{\boldwb(\RCDFpl\supb(\accrateb))}  + \lambda \partial_+ \Phi(t) \big{|}_{t = t\supa - t\supb}\right] 
	\end{bmatrix}~.
	\end{eqnarray*}
	Substituting $\Delta = t\supa - t\supb = \constfa(\accratea) - \constfb(\accrateb)$ concludes the proof. 
	\end{proof}
	In general, a closed form solution for the soft constrained problem may be difficult to state. However, for the case of $\Phi(t) = |t|$, we can state an explicit closed form solution:
	\begin{prop}[Special case of $\Phi(t) = |t|$]\label{prop:abs_val_soft_constrainted} Let $\Phi(t) = |t|$, fix $\lambda$, and let $[\accratea^{\lambda,-},\accratea^{\lambda,+}]$ denote the interval of optimal selection rates for Equation~\eqref{eq:soft_const} with regularization $\lambda$. Finally, suppose that for any optimal $\maxprof$ selection rates $(\accratea^\maxprof,\accrateb^\maxprof)$, one has $\constfa(\accratea^\maxprof) < \constfb(\accrateb^\maxprof)$.  
	Let $[\accratea^-,\accratea^+]$ denote the optimal loan rates in~\eqref{eq:soft_const}. Then there exists a $\lambda_*$ such that, for $\lambda \ge \lambda_*$,  $[\accratea^-,\accratea^+]$ coincides with the hard constrained solution. Moreover, for $\lambda < \lambda_*$, any $\accrate \in [0,1]$ satifies
	\begin{eqnarray*}
		\accrate < \accratea^{\lambda,-} &\text{if}& \fraca\frac{\util(\RCDF\supa(\accrate))}{\boldwa(\RCDF\supa(\accrate))} + \sigma_* \lambda > 0 \\
		\accrate > \accratea^{\lambda,+} &\text{if}& \fraca\frac{\util(\RCDFpl\supa(\accrate))}{\boldwa(\RCDFpl\supa(\accrate))} + \sigma_* \lambda < 0 ~.
		\end{eqnarray*}
	\end{prop}
	\begin{proof} Given a set of optimal constraint values $(t\supa, t\supb) = (\constfa(\accratea),\constfb(\accrateb))$ for optimal selection rates $(\accratea,\accrateb)$ for a given parameter $\lambda$. By Proposition~\ref{prop:soft_interpolation} below, it follows that if $t\supa = t\supb$ for all optimal solutions, then for all $\lambda' \ge \lambda$, all optimal solutions must also have $t\supa = t\supb$. 

	Hence, it suffices to show that (a) there exists a finite $\lambda$ such that all solutions must have $t\supa = t\supb$, and (b) if $t\supa \ne t\supb$, then the display in~\eqref{prop:abs_val_soft_constrainted} holds. 

	To prove (a) and (b), suppose $t\supa \ne t\supb$. By Proposition~\ref{prop:soft_interpolation} below and the fact that $\constfa(\accrate^\maxprof) < \constfb(\accrateb^\maxprof)$, we have $t\supa < t\supb$. Moreover we can compute that
	\begin{eqnarray*}
	\partial \Phi(t) = \begin{cases} \{1\} & t > 0\\
	[-1,1] & t = 0 \\
	\{-1\} & t < 0 \end{cases}
	\end{eqnarray*}
	it follows from the first order condition in Lemma~\ref{lem:first_order_op} that, if $t\supa \ne t\supb$
	\begin{eqnarray}\label{eq:first_order_absval}
	0 \in [ \frac{\util(\RCDFpl\supa(\accratea))}{\boldwa(\RCDFpl\supa(\accratea))} + \lambda  , \frac{\util(\RCDF\supa(\accratea))}{\boldwa(\RCDF\supa(\accrateb))}  + \lambda ]~,
	\end{eqnarray}
	which immediately implies point (b). Point (a) follows from the above display by noting that, since $\boldwj(x) > 0$ and $\util(x) < \infty$ for all $x$, where exists a $\lambda$ sufficiently large such that~\eqref{eq:first_order_absval} cannot hold for any $\accratea$.
	\end{proof}

\subsection{Qualitative Behavior of Soft Constraints}\label{sec:soft_constraints}
	
	We now present a proposition which formalizes the intuition that soft constraints interpolate between $\maxprof$ and the general hard constraint~\eqref{eq:hard_const_eqop}  in Section~\ref{sec:eqop_proof} (for arbitrary $\boldw$, not just for $\eqop$). Because optimal policies may not be unique, we define the solution sets 
	\begin{eqnarray*}
	\mathbf{P}(\lambda) &:=& \{ (\pola,\polb) : (\pola,\polb) \text{ solves}~\eqref{eq:soft_const} \text{ with parameter } \lambda \}~,
	\end{eqnarray*}
	with the set $\mathbf{P}(\infty)$ denoting the set of solutions to~\eqref{eq:hard_const_eqop}.

	At a high level, we parameterize the soft constrained solution in terms of the value of the constraint $t\supa =\langle \pola,\boldwa \circ \dista \rangle$ for $\popa$ and the difference in constraint values $\Delta =\langle \pola,\boldwa \circ \dista \rangle - \langle \polb,\boldwb \circ \distb \rangle$, where $(\pola,\polb) \in \mathbf{P}(\lambda)$. We show that $t\supa$ interpolates between the value of the constraint on $\popa$ at $\lambda = 0$ and at $\lambda = \infty$, and that $\Delta$ interpolates between the difference at $\lambda = 0$ ($\maxprof$) and at $\Delta = 0$ at $\lambda = \infty$. To be rigorous, we note that the possible values for $t\supa$ and $\Delta$ for each $\lambda$ are actually contiguous intervals. Hence, to make the interpolation precise, we define the following partial order on such intervals:
	\begin{defn}[Interval order] Let $\calS_1,\calS_2$ be two intervals. We say that $\calS_1~\prec~\calS_2$ if $\max~\{x \in \calS_1~\} < \min~\{x \in \calS_2\}$ and  $\calS_1~\preceq~ \calS_2$ if both $\max~\{x \in \calS_1 \}~\le~\max~\{x \in \calS_2\}$ and $\min \{x \in \calS_1 \}~\le~\min~\{x \in \calS_2\}$.	We say that an interval-valued function $\calS(\lambda)$ is \emph{non-decreasing} (resp. \emph{non increasing}) in $\lambda$ if $\calS(\lambda)~\preceq~\calS(\lambda')$ (resp $\calS(\lambda ')~\preceq~\calS(\lambda')$ for $\lambda~ \le~\lambda'$). 
	\end{defn}
	In these terms, the interpolation of the soft constraints can be stated as follows:
	\begin{prop}[Soft constraints interpolate between $\maxprof$ and hard constrained solution]\label{prop:soft_interpolation} Let $\Phi(t)$ be a convex, symmetric convex function with $\Phi(t) > 0$ for $t > 0$. Then the sets
	\begin{eqnarray*}
	\mathcal{D}(\lambda) &:=& \{\Delta := \langle \pola,\boldwa \circ \dista \rangle - \langle \polb,\boldwb \circ \distb \rangle  :(\pola,\polb) \in \mathbf{P}(\lambda)\}\\
	 \mathcal{T}\supa(\lambda) &:=& \{t\supa := \langle \pola,\boldwa \circ \dista \rangle | \exists \polb: (\pola,\polb) \in \mathbf{P}(\lambda) \}
	\end{eqnarray*}
	are closed intervals. Moreover, 
	\begin{enumerate}
	\item In all cases, $\lim_{\lambda \to \infty} \max \{|\Delta| \in \mathcal{D}(\lambda)\} = 0$.
	\item If $0 \in \mathcal{D}(\lambda) $, then there exists a $\maxprof$ solution satisfying~\eqref{eq:hard_const_eqop}. Thus, for all $\lambda > 0$, $\mathbf{P}(\lambda) = \mathbf{P}(\infty)$.
	\item If $\mathcal{D}(\lambda) \prec \{0\}$, then $\mathcal{D}(\lambda)$ and $\mathcal{T}\supa(\lambda)$ are non-decreasing on $\lambda \in (0,\infty]$, and vice versa if $\mathcal{D}(\lambda) \succ \{0\}$. 
	\item If $\mathcal{D}(\lambda) \prec \{0\}$, then $\{0\} = \mathcal{D}(\infty) \succeq \mathcal{D}(\lambda) \succeq \{\min:\Delta \in \mathcal{D}(0)\}$, and $\mathcal{T}\supa(\infty) \succeq \mathcal{T}\supa(\lambda) \succeq \{\min:\Delta \in \mathcal{T}\supa(\lambda)\}$, and vice versa if $\mathcal{D}(\lambda) \succ \{0\}$.
\end{enumerate}
	\end{prop}

\subsubsection{Proof of Proposition~\ref{prop:soft_interpolation}}

	Again, we parameterize all solutions to the soft-constrained problem as in correspondence with solutions $(t\supa,t\supb)$ to 
	\begin{eqnarray*}
	\min_{(t\supa,t\supb)} \fraca \Util\supa(t\supa;\boldwa) + \fracb \Util\supb(t\supb;\boldwb) + \lambda \Phi(t\supa - t\supb)\:.
	\end{eqnarray*}
	Letting $\Delta := t\supb - t\supa$, we can reparameterize the above as 
	\begin{eqnarray*}
	\min_{(t\supa,\Delta)} \fraca \Util\supa(t\supa;\boldwa) + \fracb \Util\supb(t\supa + \Delta;\boldwb) - \lambda \Phi(\Delta)\:.
	\end{eqnarray*}
	Note then that $\mathcal{D}(\lambda)$ denotes the set of $\Delta$ which are partial maximimizers of the above display. If $0 \in \{\mathcal{D}(\lambda)\}$, this implies that there exists a $\maxprof$ solution for which $\Delta = 0$, therefore, for all $\lambda > 0$, all solutions will be $\maxprof$ solutions for which $\mathcal{D}(\lambda) = 0$. Otherwise assume without loss of generality that $\mathcal{D}(\lambda) < \{0\}$.

	First, the statement $\{0\} = \mathcal{D}(\infty) \succeq \mathcal{D}(\lambda) \succeq \{\min:\Delta \in \mathcal{D}(0)\}$, and $\mathcal{T}\supa(\infty) \succeq \mathcal{T}\supa(\lambda) \succeq \{\min:\Delta \in \mathcal{T}\supa(\lambda)\}$, and vice versa if $\mathcal{D}(\lambda) \succ \{0\}$ can be solved by on a case-by-case basis. The strategy is to show that if any of these inequalities are violated, then the associated values of $\Delta$ and $t\supa$ are not partial maximizers of the soft constraint objective. In particular, $\mathcal{T}\supa(\lambda) \subset [T_-,T_+]$ for some appropriate $T_-,T_+$. 

	We now show that $\mathcal{D}(\lambda)$ and $\mathcal{T}\supa(\lambda)$ are non-increasing and non-decreasing, respectively. We shall do so invoking the following technical lemma.
	\begin{lem}\label{lem:composite_lem} Let $G_1(t)$ be concave and let $G_2(t;\lambda)$ be concave in $t$. Let $\partial G_2(t;\lambda)$ denote the super-gradient of $G_2$, that is 
	\begin{eqnarray*}
	\partial G_2(t;\lambda) := \mathrm{Conv}(\{\partial_- G_2(t;\lambda)\} \cup \{\partial_- G_2(t;\lambda)\})
	\end{eqnarray*}
	denotes the super-gradient set of the concave mapping $t \mapsto \partial G_2(t;\lambda)$.  

	Then if $\lambda \mapsto \partial G_2(t;\lambda)$ is non-increasing (resp. non-decreasing) in $\lambda$, the interval valued function defined below is non-increasing (resp. non-decreasing) in $\lambda$
	\begin{eqnarray*}
	\mathrm{MAX}(\lambda) := \lambda \mapsto  \arg\max_{t \in [a,b]} G_1(t) + G_2(t;\lambda)\:.
	\end{eqnarray*}
	
	\end{lem}
	For $\mathcal{D}(\lambda)$, one can write any partial maximizer $\Delta$ as 
	\begin{eqnarray*}
	\max_{\Delta \ge 0} G_1(\Delta) + G_2(\Delta;\lambda)
	\end{eqnarray*}
	with $G_1(\Delta) = \max_{t\supa} \fraca \Util\supa(t\supa;\boldwa) + \fracb \Util\supb(t\supa + \Delta;\boldwb)$ and $G_2(\Delta;\lambda) = \lambda\Phi(\Delta)$. Note that $G_1(\Delta)$ is concave, being the partial maximization of a concave function, and $\partial G_2(\Delta;\lambda) = -t\partial \Phi(\Delta)$. Since $\partial \Phi(\Delta) \succeq \{0\}$ for $\Delta \ge 0$ (by convexity of $\phi$) , we have that $\partial G_2(\Delta;\lambda) = -t\partial \Phi(\Delta)$ is non-increasing in $\lambda$. Hence Lemma~\ref{lem:composite_lem} implies that interval valued function $\mathcal{D}(\lambda)$ is non-increasing.
	
	To show that $\mathcal{T}\supa(\lambda)$ is non-decreasing, we have that any maximizer $t\supa$ can be written as 
	\begin{eqnarray*}
	\max_{t\supa \in [T_-,T_+]} G_1(t\supa) + G_2(t\supa;\lambda)
	\end{eqnarray*}
	where $G_1(t\supa) = \fraca \Util\supa(t\supa;\boldwa) $ and $ G_2(t\supa;\lambda) = \max_{\Delta \ge 0}  \fracb \Util\supb(t\supa + \Delta;\boldwb) + \lambda \Phi(\Delta)$. By Danskin's theorem,
	\begin{eqnarray*}
	\partial G_2(t\supa;\lambda) = \{\partial  \Util\supb(t\supa + \Delta;\boldwb) : \Delta \in \arg\max G_2(t\supa;\lambda) \}\:.
	\end{eqnarray*} 
	Note that $\{\Delta \in \arg\max G_2(t\supa;\lambda)\}$ is non-increasing in $\lambda$ for a fixed $t\supa$, since the contribution of the regularizer increases. Since the sets  $\partial  \Util\supb(t\supa + \Delta;\boldwb)$ are themselves non-increasing in $\Delta$ by concavity, we conclude that $\partial G_2(t\supa;\lambda)$ is non-decreasing in $\lambda$. Hence, Lemma~\ref{lem:composite_lem} implies that $\calT\supa(\lambda)$ is non-decreasing in $\lambda$.

	Finally, to show that $\max\{|\Delta|:\Delta \in \mathcal{D}(\lambda)|\} \to 0$,   	Note that the left and right derivatives of $\fraca \Util\supa(t;\boldwa)$ and $ \fracb \Util\supb(t;\boldwb)$ are upper bounded by $M$ whereas, since $\Phi$ is strictly convex, we know that for every $\epsilon > 0$, $\min\{|\partial_+ \Phi(\Delta)|,|\partial_- \Phi(\Delta)|\} > m(\epsilon)$ for all $\Delta:|\Delta| > \epsilon$. Hence, the first order optimality conditions cannot be satisfied for $|\Delta| > \epsilon$, and $\lambda > \frac{M}{m(\epsilon)}$, so as $\lambda \to \infty$, $|\Delta| \to 0$.

	\begin{proof}[Proof of Lemma~\ref{lem:composite_lem}] We prove the case where $\partial G_2(t;\lambda)$ is non-increasing. The first order conditions requires that at an optimal $t$, one has 
	\begin{eqnarray*}
	\partial_- G_1(t) + \partial G_2(t;\lambda)_- \ge 0  \ge \partial_+ G_1(t) + \partial G_2(t;\lambda)_+
	\end{eqnarray*}
	where the super-gradients are amended to take into account boundary conditions. Suppose that for the sake of contradiction that for $\lambda' > \lambda$, $\mathrm{MAX}(\lambda') \preceq \mathrm{MAX}(\lambda)$ fails. Then, there (a) exists a $t \in \mathrm{MAX}(\lambda)$ such that $\{t\} \prec \mathrm{MAX}(\lambda')$, or (b) $t \in \mathrm{MAX}(\lambda')$ such that $\{t\} \succ \mathrm{MAX}(\lambda')$. Note that if $\{t\} \prec \mathrm{MAX}(\lambda')$, it must be the case that
	\begin{eqnarray*}
	\partial_+ G_1(t) + \partial G_2(t;\lambda')_+  > 0\:.
	\end{eqnarray*}
	By assumption, $\partial_- G_2(t;\lambda')_+ \le \partial_0 G_2(t;\lambda)_+$ , which implies
	\begin{eqnarray*}
	\partial_+ G_1(t) + \partial G_2(t;\lambda')_+ \le \partial_+ G_1(t) + \partial_+ G_2(t;\lambda)_- 0 \le 0\:,
	\end{eqnarray*}
	a contradiction.
	\end{proof}

%% file: 0_appendix_main_results_proofs.tex
\section{Proofs of Main Results\label{app:main_results_proofs}}
We remark that the proofs in this section rely crucially on the characterizations of the optimal fairness-constrained policies developed in Section~\ref{sec:main_thm_proofs}.
We first define the notion of CDF domination, which is referred to in a few of the proofs. Intuitively, it means that for any score, the fraction of group $\popb$ above this is higher than that for group $\popa$. It is realistic to assume this if we keep with our convention that group $\popa$ is the disadvantaged group relative to group $\popb$.
\begin{defn}[CDF domination]
	$\dista$ is said to be \emph{dominated by} $\distb$ if $\forall a \ge 1, \sum_{x>a}\dist\supa < \sum_{x>a}\dist\supb$. We denote this as $\dista \prec \distb$.
\end{defn}
We remark that the $\prec$ notation in this section is entirely unrelated to the the partial order on intervals from Section~\ref{sec:soft_constraints}. Frequently, we shall use the following lemma:
\begin{lem} Suppose that $\dista \prec \distb$. Then, for all $\accrate > 0$, it holds that $\RCDF\supa(\beta) \le \RCDF\supb(\accrate)$ and $\util(\RCDF\supa(\accrate)) \le \util(\RCDF\supa(\accrate))$
\end{lem}
\begin{proof}
The fact that $\RCDF\supa(\accrate) \le \RCDF\supb(\accrate)$ follows directly from the definition of monotonicty of $\util$ implies that $\util(\RCDF\supa(\accrate)) \le \util(\RCDF\supb(\accrate))$. 
\end{proof}

\subsection{Proof of Proposition \ref{prop:mp_noharm}}
	The $\maxprof$ policy for group~$\popj$ solves the optimization
	\begin{eqnarray*}
	\max_{\polj \in [0,1]^{\maxscore}} \Util\supj(\polj) = \max_{\accratej \in [0,1]} \Util\supj(\ratef_{\distj}^{-1}(\accratej))\:.
	\end{eqnarray*}
	Computing left and right derivatives of this objective yields
	\begin{align*}
		\partial_+ \Util\supj(\ratef_{\distj}^{-1}(\accratej)) &= \util(\RCDF\supj(\beta)), \qquad \partial_- \Util\supj(\ratef_{\distj}^{-1}(\accratej)) = \util(\RCDFpl\supj(\beta))\:.
	\end{align*}
		By concavity, solutions $\accrate^*$ satisfy
		\begin{align}
		\begin{split}\label{eq:max_prof_sol_char}
		\accrate < \accrate^* \quad&\text{if}\quad \util(\RCDF\supj(\beta)) > 0 \:,\\
		\accrate > \accrate^* \quad&\text{if}\quad \util(\RCDFpl\supj(\beta)) < 0 \:.
		\end{split}
		\end{align}
	Therefore, we conclude that the $\maxprof$ policy loans only to scores $x$ s.t. $\util(x) > 0$, which implies $\chg(x) > 0$ for all scores loaned to. Therefore we must have that $ 0\leq \delmean^{\maxprof} $. By definition $\delmean^{\maxprof} \leq \delmean^*$.
	
\subsection{Proof of Corollary \ref{cor:fairness_rel_imp} }
	
		We begin with proving part (a), which gives conditions under which $\dempar$ cases relative improvement. Recall that $\overline{\accrate}$ is the largest selection rate for which $\Util(\overline\accrate) = \Util(\accratea^\maxprof)$. First, we derive a condition which bounds the selection rate $\beta^\dempar\supa$ from below. Fix an acceptance rate $\accrate$ such that $\accrate\supa^{\maxprof}<\accrate<\min\{\accrate\supb^{\maxprof}, \overline{\accrate}\}$. By Theorem~\ref{thm:dem_par_selection}, we have that $\dempar$ selects to group~$\popa$ with rate higher than $\accrate$ as long as
		\begin{eqnarray*}
		\fraca \leq g_1 := \frac{1}{1-\frac{\util(\RCDF\supa(\accrate))}{\util(\RCDF\supb(\accrate))}}\:.
		\end{eqnarray*}
		By \eqref{eq:max_prof_sol_char} and the monotonicity of $\util$, $\util(\RCDF\supa(\accrate))<0$ and $\util(\RCDF\supb(\accrate))>0$, so $0<g_1<1$. 

		Next, we derive a condition which bounds the selection rate $\accrate^\dempar\supa$ from above. First, consider the case that $\accrate\supb^{\maxprof}<\overline{\accrate}$, and fix $\accrate'$ such that $\accrate\supb^{\maxprof}<\accrate'<\overline\accrate$. 
		Then $\dempar$ selects group~$\popa$ at a rate $\accratea < \accrate'$ for any proportion $\fraca$. This follows from applying Theorem~\ref{thm:dem_par_selection} since we have that  $\util(\RCDFpl\supa(\accrate'))<0$ and $\util(\RCDFpl\supb(\accrate'))<0$ by \eqref{eq:max_prof_sol_char} and the monotonicity of $\util$. 

		Instead, in the case that $\accrate\supb^{\maxprof}>\overline{\accrate}$, fix $\accrate'$ such that $\overline\accrate<\accrate'<\accrate\supb^{\maxprof}$. Then $\dempar$ selects group~$\popa$ at a rate less than $\accrate'$ as long as 
		\begin{align*}
		\fraca\geq g_0:=\frac{1}{1-\frac{\util(\RCDFpl\supa(\accrate'))}{\util(\RCDFpl\supb(\accrate'))}}\:.
		\end{align*} 
		By \eqref{eq:max_prof_sol_char} and the monotonicity of $\util$, $0<g_0<g_1$. Thus for $\fraca\in[g_0,g_1]$, the $\dempar$ selection rate for group~$\popa$ is bounded between $\accrate$ and $\accrate'$, and thus $\dempar$ results in relative improvement.

		Next, we prove part (b), which gives conditions under which $\eqop$ cases relative improvement. First, we derive a condition which bounds the selection rate $\accrate^\eqop\supa$ from below. Fix an acceptance rate $\accrate$ such that $\accrate\supa^{\maxprof}<\accrate$ and $\accrate\supb^{\maxprof} > \transfer(\accrate)$.  
		By Theorem~\ref{thm:eqop_selection}, $\eqop$ selects group~$\popa$ at a rate higher than $\accrate$ as long as
		\begin{align*}
		\fraca>g_3 := \frac{1}{1-\frac{1}{\kappa}\cdot\frac{\pb(\RCDF\supb(\transfer(\accrate)))}{\util(\RCDF\supb(\transfer(\accrate)))}\frac{\util(\RCDF\supa(\accrate))}{\pb(\RCDF\supa(\accrate))}}\:.
		\end{align*} 
		By \eqref{eq:max_prof_sol_char} and the monotonicity of $\util$, $\util(\RCDF\supa(\accrate))<0$ and $\util(\RCDF\supb(\transfer(\accrate)))>0$, so $g_3>0$.

		Next, we derive a condition which bounds the selection rate $\accrate^\eqop\supa$ from above. First, consider the case that there exists $\accrate'$ such that $\accrate'<\overline\accrate$ and $\accrate\supb^{\maxprof} < \transfer(\accrate')$ .
		Then $\eqop$ selects group~$\popa$ at a rate less than this $\accrate'$ for any $\fraca$. This follows from Theorem~\ref{thm:eqop_selection} since we have that  $\util(\RCDFpl\supa(\accrate'))<0$ and $\util(\RCDFpl\supb(\transfer(\accrate')))<0$ by \eqref{eq:max_prof_sol_char} and the monotonicity of $\util$. 

		In the other case, fix $\accrate'$ such that $\accrate<\accrate'<\overline\accrate$ and $\accrate\supb^{\maxprof} > \transfer(\accrate')$. By Theorem~\ref{thm:eqop_selection}, $\eqop$ selects group~$\popa$ at a rate lower than $\accrate'$ as long as
		$$\fraca>g_2 := \frac{1}{1-\frac{1}{\kappa}\cdot\frac{\pb(\RCDFpl\supb(\transfer(\accrate')))}{\util(\RCDFpl\supb(\transfer(\accrate')))}\frac{\util(\RCDFpl\supa(\accrate'))}{\pb(\RCDFpl\supa(\accrate'))}}\:.$$

		By \eqref{eq:max_prof_sol_char} and the monotonicity of $\util$, $0<g_2<g_3$. Thus for $\fraca\in[g_2,g_3]$, the $\eqop$ selection rate for group~$\popa$ is bounded between $\accrate$ and $\accrate'$, and thus $\eqop$ results in relative improvement.

	\subsection{Proof of Corollary \ref{cor:dp_overeager}}
	
	Recall our assumption that $\accrate > \accrate\supa^{\maxprof}$ and  $\accrate\supb^{\maxprof} > \accrate$.
	As argued in the above proof of Corollary \ref{cor:fairness_rel_imp}, by \eqref{eq:max_prof_sol_char} and the monotonicity of $\util$, $\util(\RCDF\supa(\accrate))<0$ and $\util(\RCDF\supb(\accrate))>0$. 
	 Applying Theorem~\ref{thm:dem_par_selection}, $\dempar$ selects at a higher rate than $\accrate$ for any population proportion $g_\popa \leq g_0$, where $g_0 = 1/(1-\frac{\util(\RCDF\supa(\accrate))}{\util(\RCDF\supb(\accrate))}) \in (0,1)$. In particular, if $\accrate=\accrate_0$, which we defined as the harm threshold (i.e. $\delmean_\popa(\ratef^{-1}_{\dista} (\accrate_0))=0$ and $\delmean_\popa$ is decreasing at $\accrate_0$), then by the concavity of $\delmean_\popa$, we have that  $\delmean_\popa(\ratef^{-1}_{\dista}(\accrate^\dempar_\popa))< 0$, that is, $\dempar$ causes active harm.

\subsection{Proof of Corollary \ref{cor:eqop_overeager}}

		By Theorem~\ref{thm:eqop_selection}, $\eqop$ selects at a higher rate than $\accrate$ for any population proportion $g_\popa \leq g_0$, where $g_0 =  1/(1-\frac{1}{\kappa}\cdot\frac{\pb(\RCDF\supb(\transfer(\accrate)))}{\util(\RCDF\supb(\transfer(\accrate)))}\frac{\util(\RCDF\supa(\accrate))}{\pb(\RCDF\supa(\accrate))})$. Using our assumptions $\accrate\supb^{\maxprof} > \transfer(\accrate)$ and $\accrate > \accrate\supa^{\maxprof}$, we have that $\util(\RCDF\supb(\transfer(\accrate))) > 0$ and $\util(\RCDF\supa(\accrate)) < 0$, by \eqref{eq:max_prof_sol_char} and the monotonicity of $\util$. This verifies that $g_0 \in (0,1)$. In particular, if $\accrate=\accrate_0$, then by the concavity of $\delmean_\popa$, we have that  $\delmean_\popa(\ratef^{-1}_{\dista}(\accrate^\eqop_\popa))< 0$, that is, $\eqop$ causes active harm.

\subsection{Proof of Corollary \ref{cor:avoid_harm} }
	
	Applying Theorem~\ref{thm:dem_par_selection}, we have
		\[-\frac{1-\fraca}{\fraca} \util(\RCDFone(\accrate))  < \util(\RCDFtwo(\accrate)) \implies \accrate_{\dempar} > \accrate\:. \]
		Applying Theorem~\ref{thm:eqop_selection}, we have:	\[\util(\RCDFtwo(\transfer(\accrate)))\cdot \frac{\langle \pb, \distb\rangle}{ \langle \pb, \dista\rangle} \cdot  \frac{\pb(\RCDF^+\supa(\accrate))}{\pb(\RCDF^+\supb(\transfer(\accrate)))} < -\frac{1-\fraca}{\fraca} \util(\RCDFone^+(\accrate)) \implies \accrate_{\eqop} < \accrate\:. \] 
	By Corollaries~\ref{cor:dp_overeager} and \ref{cor:eqop_overeager}, choosing $\fraca < g_2:= 1/(1-\frac{\util(\RCDF\supa(\accrate))}{\util(\RCDF\supb(\accrate))})$ and $\fraca > g_1 :=1/(1-\frac{1}{\kappa}\cdot\frac{\pb(\RCDF^+\supb(\transfer(\accrate)))}{\util(\RCDF^+\supb(\transfer(\accrate)))}\frac{\util(\RCDF^+\supa(\accrate))}{\pb(\RCDF^+\supa(\accrate))})$ satisfies the above.
	
	It remains to check that $g_1 < g_2$. Since we assumed $\accrate > \sum_{x > \mean \supa } \dist\supa$, we may apply Lemma~$\ref{cor:compare_dp_eo}$ to verify this.
	
	Thus we indeed have sufficient conditions for $\accrate_{\dempar} > \accrate > \accrate_\eqop$. In particular, if $\accrate=\accrate_0$, then by the concavity of $\delmean_\popa$, we have that  $\delmean_\popa(\ratef^{-1}_{\dista}(\accrate^\eqop_\popa))> 0$, that is, $\eqop$ causes improvement, and $\delmean_\popa(\ratef^{-1}_{\dista}(\accrate^\dempar_\popa))< 0$, that is, $\dempar$ causes active harm.

	Lastly, because $\accrate_{\dempar} >  \accrate_\eqop$, it is always true that $\delmean_\popa(\ratef^{-1}_{\dista}(\accrate^\dempar_\popa)) > 0 \implies \delmean_\popa(\ratef^{-1}_{\dista}(\accrate^\eqop_\popa))> 0$, using the concavity of the outcome curve.

	\begin{lem}[Comparison of $\dempar$ and $\eqop$ selection rates]\label{cor:compare_dp_eo}
		
		Fix $\accrate \in[0,1]$. Suppose $\dist\supa, \dist\supb$ are identical up to a translation with $\mean\supa < \mean\supb$. Also assume $\pb(x)$ is affine in $x$. Denote $\kappa =\frac{\langle \pb, \distb\rangle}{ \langle \pb, \dista\rangle}$.
		Then, \[\accrate > \sum_{x > \mean \supa } \dist\supa \]
		implies $\util(\RCDFtwo(\transfer(\accrate)))\cdot\kappa \cdot  \frac{\pb(\RCDF\supa(\accrate))}{\pb(\RCDF\supb(\transfer(\accrate)))}  < \util(\RCDFtwo(\accrate))$.
		\begin{proof}
			
			If we have $\accrate > \sum_{x > \mean \supa } \dist\supa$, by lemma \ref{lem:geometry_shifted_dist}, we must also have $\frac{\mean\supb}{\mean\supa} < \frac{\RCDFtwo(\accrate_0)}{\RCDFone(\accrate_0)}$. This implies $\kappa=\frac{\sum_x\dist\supb(x)\pb(x)}{\sum_x\dist\supa(x)\pb(x)} <\frac{\pb(\RCDFtwo(\accrate))}{\pb(\RCDFone(\accrate_0))}  $ by linearity of expectation and linearity of $\pb$. Therefore, 
			\begin{equation}
			\kappa \cdot \frac{\pb(\RCDFone(\accrate))}{\pb(\RCDFtwo(\accrate_0))} <1 \label{eqn:kappa_geom}
			\end{equation} 
			
			Further, using $\transfer(\accrate) > \accrate$ from lemma \ref{lem:geometry_shifted_dist} and the fact that $\frac{\util(x)}{\pb(x)}$ is increasing in $x$, we have $\frac{\util(\RCDFtwo(\transfer(\accrate)))}{ \pb(\RCDFtwo(\transfer(\accrate)))}<\frac{\util(\RCDFtwo(\accrate))}{ \pb(\RCDFtwo(\accrate))}  $. Therefore, $ \util(\RCDFtwo(\transfer(\accrate)))\cdot \kappa \cdot  \frac{\pb(\RCDF\supa(\accrate_0))}{\pb(\RCDF\supb(\transfer(\accrate_0)))} < \kappa \cdot \frac{\util(\RCDFtwo(\accrate))}{ \pb(\RCDFtwo(\accrate))}  \cdot \pb(\RCDF\supa(\accrate)) <\util(\RCDFtwo(\accrate))$ where the last inequality follows from \eqref{eqn:kappa_geom}.
			
		\end{proof}
	\end{lem}
	
	We use the following technical lemma in the proof of the above lemma.
	
	\begin{lem}\label{lem:geometry_shifted_dist}
		If $\dist\supa, \dist\supb$ that are identical up to a translation with $\mean\supa < \mean\supb$, then 
		\begin{align}
		G(\accrate)&> \accrate ~~\forall~ \accrate\:, \label{eqn:shifted_dist_claim1}\\
		\accrate &> \sum_{x > \mean } \dist\supa \implies \frac{\mean\supb}{\mean\supa} < \frac{\RCDFtwo(\accrate)}{\RCDFone(\accrate)} \:.\label{eqn:shifted_dist_claim2}
		\end{align}
		\begin{proof}
			For \eqref{eqn:shifted_dist_claim1}, observe that $\tpr\supa = \pb(\mean\supa) < \tpr\supb = \pb(\mean\supb)$. For any $\accrate$, we can write $\RCDFtwo(\accrate) = \mean\supb + c$ and $\RCDFone(\accrate) = \mean\supa + c$ for some $c$, since $\dist\supa, \dist\supb$ that are identical up to translation by $\mean\supa -\mean\supb$. Thus, by computation, we can see that for $\RCDF(\beta) < \mean$, $\partial_+\transfer(\beta) > 1$ and for $\RCDF(\beta) < \mean$, $\partial_+\transfer(\beta) < 1$. Since $\transfer$ is monotonically increasing on $[0,1]$, we must have $\transfer(\accrate) > \accrate$ for every $\accrate \in[0,1]$.
			
			For \eqref{eqn:shifted_dist_claim2}, we have $\accrate> \sum_{x > \mean } \dist\supa$, we can again write $\RCDFtwo(\accrate) = \mean\supb - c$ and $\RCDFone(\accrate) = \mean\supa - c$, for some $c > 0$. Then it is clear than we have $\frac{\mean\supb}{\mean\supa} < \frac{\RCDFtwo(\accrate)}{\RCDFone(\accrate)} $.
		\end{proof}
		
	\end{lem}
		
\subsection{Proof of Corollary \ref{cor:eqop_underloan}}
	\begin{proof}
		$\accrate^\maxprof\supa < \accrate^\maxprof\supb$ implies $\fraca \cdot  \util(\RCDF\supa(\accrate^\maxprof\supa)) + \fracb \cdot \util(\RCDF\supb(\accrate^\maxprof\supa)) > 0$, which by Theorem~\ref{thm:dem_par_selection}, implies $\accrate\supa^\maxprof < \accrate\supa^\dempar$.
		
		$\oppa(\pol^\maxprof) > \oppb(\pol^\maxprof)$ implies	$\transfer(\accrate^\maxprof\supa) >\accrate^\maxprof\supb$ and so
		
		 $\util(\RCDF\supb(\transfer(\accrate^\maxprof\supa)))<0$. 
		 Therefore by Theorem~~\ref{thm:eqop_selection}, we have that $\accrate^\maxprof\supa > \accrate^\eqop\supa$.
	\end{proof}
	We now give a very simple example of $\dista\prec \distb$ where Theorem 3.5 holds. The construction of the example exemplifies the more general idea of using large in-group inequality in group $\popa$ to skew the true positive rate at $\maxprof$, making $\oppa(\pol^\maxprof) > \oppb(\pol^\maxprof)$.
	
	\begin{exmp}[$\eqop$ causes relative harm]
		Let $C=6$, and let the utility function be such that 
		$\util(4) = 0$. 
		Suppose $\dist\supa(5) = 1-2\epsilon, \dist\supa(1) =  2\epsilon$ and $\dist\supb(5) = 1-\epsilon, \dist\supb(3) = \epsilon$.
		
		We can easily check that $\dista\prec \distb$. However, for any $\epsilon \in (0,1/4)$, we have that $\oppb(\pol^\maxprof) = \frac{5(1-\epsilon)}{5(1-\epsilon)+3\epsilon} < \oppa(\pol^\maxprof) = \frac{5(1-2\epsilon)}{5(1-2\epsilon)+2\epsilon}$.
		 
	\end{exmp}

	 \subsection{Proof of Proposition \ref{prop:underestim}}
	 
	 Denote the upper quantile function under $\widehat{\dist}$ as $\widehat{\RCDF}$.
	 Since $\widehat{\dist} \prec \dist$, we have $\widehat{\RCDF}(\beta)\leq\RCDF(\beta)$. The conclusion follows for $\maxprof$ and $\dempar$ from Theorem~\ref{thm:dem_par_selection} by the monotonicity of $\util$.
	 
	If we have that $\oppa(\pol) > \widehat{\oppa}(\pol)~\forall~\tau$, that is, the true TPR dominates estimated TPR, the conclusion for $\eqop$ follows from Theorem~\ref{thm:eqop_selection}, by the same argument as in the proof of Corollary~\ref{cor:eqop_underloan}.
	 
	\subsection{Proof of Proposition \ref{prop:outcome}}

	 By Proposition \ref{prop:beta_concavity}, $\accrate^*=\argmax_\accrate \delmean\supa(\accrate)$ exists and is unique. $\accrate_0=\max\{\accrate \in [\accrate\supa^\maxprof,1]:\Util(\accrate\supa^{\maxprof}) - \Util\supa(\accrate) \leq \delta \}$  which exists and is unique, by the continuity of $\delmean\supa$ and Proposition \ref{prop:beta_concavity}.
	

%% file: writeup.bbl
\begin{thebibliography}{20}
\providecommand{\natexlab}[1]{#1}
\providecommand{\url}[1]{\texttt{#1}}
\expandafter\ifx\csname urlstyle\endcsname\relax
  \providecommand{\doi}[1]{doi: #1}\else
  \providecommand{\doi}{doi: \begingroup \urlstyle{rm}\Url}\fi

\bibitem[Barocas and Selbst(2016)]{barocasselbst16}
Solon Barocas and Andrew~D. Selbst.
\newblock Big data's disparate impact.
\newblock \emph{California Law Review}, 104, 2016.

\bibitem[Calders et~al.(2009)Calders, Kamiran, and Pechenizkiy]{Calders2009}
Toon Calders, Faisal Kamiran, and Mykola Pechenizkiy.
\newblock Building classifiers with independency constraints.
\newblock In \emph{Proc.~IEEE ICDMW}, ICDMW '09, pages 13--18, 2009.

\bibitem[Chouldechova(2016)]{chould16fair}
Alexandra Chouldechova.
\newblock Fair prediction with disparate impact: A study of bias in recidivism
  prediction instruments.
\newblock \emph{FATML}, 2016.

\bibitem[Ensign et~al.(2017)Ensign, Friedler, Neville, Scheidegger, and
  Venkatasubramanian]{ensign2017runaway}
Danielle Ensign, Sorelle~A Friedler, Scott Neville, Carlos Scheidegger, and
  Suresh Venkatasubramanian.
\newblock Runaway feedback loops in predictive policing.
\newblock \emph{arXiv preprint arXiv:1706.09847}, 2017.

\bibitem[{Executive Office of the President}(2016)]{whitehouse16}
{Executive Office of the President}.
\newblock Big data: A report on algorithmic systems, opportunity, and civil
  rights.
\newblock Technical report, White House, May 2016.

\bibitem[Foster and Vohra(1992)]{foster1992economic}
Dean~P Foster and Rakesh~V Vohra.
\newblock An economic argument for affirmative action.
\newblock \emph{Rationality and Society}, 4\penalty0 (2):\penalty0 176--188,
  1992.

\bibitem[Fuster et~al.(2017)Fuster, Goldsmith-Pinkham, Ramadorai, and
  Walther]{fuster2017predictably}
Andreas Fuster, Paul Goldsmith-Pinkham, Tarun Ramadorai, and Ansgar Walther.
\newblock Predictably unequal? the effects of machine learning on credit
  markets.
\newblock \emph{SSRN}, 2017.

\bibitem[Hardt et~al.(2016)Hardt, Price, and Srebro]{hardt16equality}
Moritz Hardt, Eric Price, and Nati Srebro.
\newblock Equality of opportunity in supervised learning.
\newblock In \emph{Proc.~$30$th~NIPS}, 2016.

\bibitem[Hu and Chen(2018)]{hu18shortterm}
Lily Hu and Yiling Chen.
\newblock A short-term intervention for long-term fairness in the labor market.
\newblock In \emph{Proc.~$27$th WWW}, 2018.

\bibitem[Joseph et~al.(2016)Joseph, Kearns, Morgenstern, and
  Roth]{roth2016bandits}
Matthew Joseph, Michael Kearns, Jamie~H Morgenstern, and Aaron Roth.
\newblock Fairness in learning: Classic and contextual bandits.
\newblock In \emph{Proc.~$30$th NIPS}, pages 325--333, 2016.

\bibitem[Kalev et~al.(2006)Kalev, Dobbin, and Kelly]{Kaley2006}
Alexandra Kalev, Frank Dobbin, and Erin Kelly.
\newblock {Best Practices or Best Guesses? Assessing the Efficacy of Corporate
  Affirmative Action and Diversity Policies}.
\newblock \emph{American Sociological Review}, 71\penalty0 (4):\penalty0
  589--617, 2006.

\bibitem[Keith et~al.(1985)Keith, Bell, Swanson, and Williams]{keith1985}
Stephen~N. Keith, Robert~M. Bell, August~G. Swanson, and Albert~P. Williams.
\newblock Effects of affirmative action in medical schools.
\newblock \emph{New England Journal of Medicine}, 313\penalty0 (24):\penalty0
  1519--1525, 1985.

\bibitem[Kilbertus et~al.(2017)Kilbertus, Rojas{-}Carulla, Parascandolo, Hardt,
  Janzing, and Sch{\"{o}}lkopf]{Kilbertus17}
Niki Kilbertus, Mateo Rojas{-}Carulla, Giambattista Parascandolo, Moritz Hardt,
  Dominik Janzing, and Bernhard Sch{\"{o}}lkopf.
\newblock Avoiding discrimination through causal reasoning.
\newblock In \emph{In Proc.~$30$th NIPS}, pages 656--666, 2017.

\bibitem[Kleinberg et~al.(2017)Kleinberg, Mullainathan, and
  Raghavan]{kleinberg16inherent}
Jon~M. Kleinberg, Sendhil Mullainathan, and Manish Raghavan.
\newblock Inherent trade-offs in the fair determination of risk scores.
\newblock \emph{Proc.~$8$th ITCS}, 2017.

\bibitem[Kusner et~al.(2017)Kusner, Loftus, Russell, and Silva]{Kusner17}
Matt~J. Kusner, Joshua~R. Loftus, Chris Russell, and Ricardo Silva.
\newblock Counterfactual fairness.
\newblock In \emph{In Proc.~$30$th NIPS}, pages 4069--4079, 2017.

\bibitem[Nabi and Shpitser(2017)]{Nabi17}
Razieh Nabi and Ilya Shpitser.
\newblock Fair inference on outcomes.
\newblock \emph{arXiv:1705.10378v1}, 2017.

\bibitem[Pleiss et~al.(2017)Pleiss, Raghavan, Wu, Kleinberg, and
  Weinberger]{weinberger2017calibration}
Geoff Pleiss, Manish Raghavan, Felix Wu, Jon Kleinberg, and Kilian~Q
  Weinberger.
\newblock On fairness and calibration.
\newblock In \emph{Advances in Neural Information Processing Systems 30}, pages
  5684--5693, 2017.

\bibitem[Ross and Yinger(2006)]{ross2006}
Stephen Ross and John Yinger.
\newblock \emph{{The Color of Credit: Mortgage Discrimination, Research
  Methodology, and Fair-Lending Enforcement}}.
\newblock MIT Press, Cambridge, 2006.

\bibitem[{US Federal Reserve}(2007)]{fed07}
{US Federal Reserve}.
\newblock Report to the congress on credit scoring and its effects on the
  availability and affordability of credit, 2007.

\bibitem[Zafar et~al.(2017)Zafar, Valera, Rogriguez, and Gummadi]{zafar17a}
Muhammad~Bilal Zafar, Isabel Valera, Manuel~Gomez Rogriguez, and Krishna~P.
  Gummadi.
\newblock {Fairness Constraints: Mechanisms for Fair Classification}.
\newblock In \emph{Proc.~$20$th AISTATS}, pages 962--970. PMLR, 2017.

\end{thebibliography}
